\documentclass[compsoc, conference, a4paper, 10pt, times]{IEEEtran}

\usepackage[utf8]{inputenc}
\usepackage{amsmath}
\usepackage{amssymb}
\usepackage{amsthm}
\usepackage{xcolor}
\usepackage{mathtools}
\usepackage{soul}
\usepackage{url}  
\usepackage{algorithm}
\usepackage[noend]{algpseudocode}
\usepackage{hyperref}

\usepackage{enumitem}
\usepackage{tabularx}
\usepackage{graphicx}
\usepackage{subcaption}
\usepackage{indentfirst}
\usepackage{booktabs}
\usepackage{float}
\usepackage{xspace}
\usepackage{multirow}
\usepackage{tcolorbox}

\usepackage[style=ieee, backend=biber]{biblatex}

\newtheorem{definition}{Definition}
\newtheorem{lemma}{Lemma}

\newtheorem{corollary}{Corollary}

\newtheorem{assumption}{Assumption}

\DeclareMathOperator*{\argmin}{arg\,min}

\bibliography{main}

\newif\ifdraft
\draftfalse
\usepackage{xparse}
\usepackage{xspace}

\makeatletter
\newcommand\@newaddpunct[1]{\ifnum\spacefactor>\@m \else#1\fi}
\newcommand{\para}[1]{\noindent\textbf{#1\unskip\@newaddpunct{.}}~~}
\makeatother

\newcommand{\adv}{\mathcal{A}}

\newcommand{\ustep}{\mathcal{U}}
\newcommand{\prover}{\mathcal{T}}

\newcommand{\fingerprint}[1][]{}
\NewDocumentCommand{\pw}{o}{
  \IfNoValueTF{#1}
    {\mathcal{P}}
    {\mathcal{P}(#1)}%
}
\newcommand{\prf}{PoL\xspace}
\MakeRobust\prf

\newcommand{\pol}{PoL\xspace}
\MakeRobust\pol

\newcommand{\rnaattack}{Adversarial RNA Step\xspace}
\MakeRobust\rnaattack

\newcommand{\pols}{PoLs\xspace}
\MakeRobust\pols

\NewDocumentCommand{\oracle}{o}{
  \IfNoValueTF{#1}
    {Or}
    {Or(#1)}%
}
\NewDocumentCommand{\verifier}{o}{
  \IfNoValueTF{#1}
    {\mathcal{V}}
    {\mathcal{V}(#1)}%
}

\newcommand{\redst}[1][115pt]{\bgroup\markoverwith {\textcolor{red}{\makebox[0pt][l]{\rule[0.5ex]{#1}{0.4pt}}\rule[1ex]{#1}{0.4pt}}}\ULon}

\newcommand{\ie}{{i.e.,}\@\xspace}
\newcommand{\eg}{{e.g.,}\@\xspace}
\newcommand{\etal}{{et al.}\@\xspace}

\newcommand{\zhang}{Zhang \etal~\cite{zhang2021adversarial}\xspace}
\newcommand{\jia}{Jia \etal~\cite{jia2021proof}\xspace}
\newcommand{\re}{$\varepsilon_{repr}$\xspace} %
\newcommand{\nre}{$||\varepsilon_{\text{repr}}||$\xspace} %
\newcommand{\rd}{$d_{\text{ref}}$\xspace} %

\newcommand{\pWeights}{\mathbb{W}}
\newcommand{\pBIndex}{\mathbb{I}}
\newcommand{\pSign}{\mathbb{H}}
\newcommand{\pMeta}{\mathbb{M}}

\algnewcommand\algorithmicoptional{\textbf{Optional:}}
\algnewcommand\Optional{\item[\algorithmicoptional]}%

\newcommand\blfootnote[1]{%
  \begingroup
  \begin{NoHyper}
  \renewcommand\thefootnote{}\footnote{#1}%
  \addtocounter{footnote}{-1}%
  \end{NoHyper}
  \endgroup
}

\usepackage{tikz}
\usepackage{listofitems} %
\usetikzlibrary{arrows.meta} %
\usepackage[outline]{contour} %
\contourlength{1.4pt}

\usetikzlibrary{patterns}
\usetikzlibrary{decorations.markings}
\usetikzlibrary{positioning}
\usetikzlibrary{calc}
\usetikzlibrary{hobby}

\usepackage{dsfont}

\tikzset{>=latex} %
\usepackage{xcolor}
\colorlet{myred}{red!80!black}
\colorlet{myblue}{blue!80!black}
\colorlet{mygreen}{green!60!black}
\colorlet{myorange}{orange!70!red!60!black}
\colorlet{mydarkred}{red!30!black}
\colorlet{mydarkblue}{blue!40!black}
\colorlet{mydarkgreen}{green!30!black}
\tikzstyle{node}=[thick,circle,draw=myblue,minimum size=22,inner sep=0.5,outer sep=0.6]
\tikzstyle{node in}=[node,green!20!black,draw=mygreen!30!black,fill=mygreen!25]
\tikzstyle{node hidden}=[node,blue!20!black,draw=myblue!30!black,fill=myblue!20]
\tikzstyle{node convol}=[node,orange!20!black,draw=myorange!30!black,fill=myorange!20]
\tikzstyle{node out}=[node,red!20!black,draw=myred!30!black,fill=myred!20]
\tikzstyle{connect}=[thick,mydarkblue] %
\tikzstyle{connect arrow}=[-{Latex[length=4,width=3.5]},thick,mydarkblue,shorten <=0.5,shorten >=1]
\tikzset{ %
	node 1/.style={node in},
	node 2/.style={node hidden},
	node 3/.style={node out},
}

\newcommand{\cat}[1]{\;\operatorname{\big\|}\;}

\usepackage{ulem}

\newcommand{\del}[1]{}
\usepackage{scrextend}
\deffootnote[0em]{0em}{0em}{\textsuperscript{\thefootnotemark}\,}

\begin{document}

\title{Proof-of-Learning is Currently More Broken Than You Think}

\author{Congyu Fang\text{*}\IEEEauthorrefmark{3}\IEEEauthorrefmark{4}, Hengrui Jia\text{*}\IEEEauthorrefmark{3}\IEEEauthorrefmark{4}, Anvith Thudi\IEEEauthorrefmark{3}\IEEEauthorrefmark{4}, Mohammad Yaghini\IEEEauthorrefmark{3}\IEEEauthorrefmark{4}, \\Christopher A. Choquette-Choo\IEEEauthorrefmark{5}, Natalie Dullerud\IEEEauthorrefmark{3}\IEEEauthorrefmark{4}, Varun Chandrasekaran\IEEEauthorrefmark{2}, Nicolas Papernot\IEEEauthorrefmark{3}\IEEEauthorrefmark{4}\vspace*{0.15cm} \\  University of Toronto\IEEEauthorrefmark{3}, Vector Institute\IEEEauthorrefmark{4}, Microsoft Research\IEEEauthorrefmark{2}, Google Research (Brain Team) \IEEEauthorrefmark{5}}
\maketitle

\begin{abstract}

Proof-of-Learning (PoL) proposes that a model owner log{s} training checkpoints %
to establish a proof of having expended the {computation} %
necessary for training. 
The authors of PoL forego cryptographic approaches and trade rigorous security guarantees
for scalability to deep learning.
They empirically argued the benefit of this approach by showing how spoofing%
{---computing a proof for a stolen model---}%
is as expensive as  obtaining the proof honestly by training the model. 
However, recent work {has} provided a counter-example and {thus has} invalidated this observation.

\looseness=-1
{In this work we demonstrate, {first,} that while it is true that current PoL verification is not robust to adversaries, recent work has largely underestimated this lack of robustness. This is because existing spoofing strategies are either unreproducible or target weakened instantiations of PoL---meaning they are easily thwarted by changing hyperparameters of the verification. Instead, we introduce the first spoofing strategies that can be reproduced across different configurations of the PoL verification and can be done for {a} fraction of the cost of previous spoofing strategies. This is possible because we identify key 
vulnerabilities of PoL
and systematically analyze the underlying assumptions needed for robust verification of a proof.
On the theoretical side, we show how realizing these assumptions reduces to open problems in learning theory. 
We conclude that one cannot develop a provably robust PoL verification mechanism without further understanding of optimization in deep learning.}
\blfootnote{\text{*}Equal Contribution.}

\end{abstract}

\section{Introduction}
\label{sec:intro}

Scaling verified computing to deep learning is difficult due to the inability to express optimization problems in a format amenable for {efficiently} generating cryptographic proofs~\cite{ozdemir2020unifying}. Thus, relatively little progress has been made in applying cryptographic primitives to {attesting the integrity} of training algorithms. 

To address this bottleneck, \jia explored non-cryptographic {approaches for} verifying {if an entity executed stochastic gradient descent {(SGD)}}---the canonical  algorithm for training deep neural networks (DNNs). {The said entity would then} obtain a {``proof''} of computation expended towards training. They propose the Proof-of-Learning (\pol) protocol for a prover to attest to the integrity of a training run by logging the intermediate states achieved by the learner. We describe the information that is logged in {Section}~\ref{sec:background}, {but} for now, it is sufficient for the reader to assume that the log contains the DNN's weights after each step of gradient descent, \ie the {\em training trajectory}. This log becomes a \textit{proof}, which can be verified by an external party; the verifier simply needs to repeat (some of) the training steps logged in the proof ({a.k.a.} duplicated execution) and compare the reproduced steps to the ones logged. This comparison takes place in the {\em weight space}. If this verification succeeds, the proof constitutes evidence that the prover ran the training algorithm correctly and obtained the model legitimately. 

While the \pol protocol is simple, \jia  were unable to provide a formal analysis of security for the proof verification mechanism they {proposed}. 
Instead, they demonstrated empirically that several strategies to find {\em spoofs}---proofs that {successfully} pass verification but were not obtained honestly by training the model---did not give the adversary an advantage. However, \zhang {contradict this} claim. They propose a strategy that would give the adversary an advantage by returning a proof that passes verification at a {lower} computational cost.

In this work, we reconcile these conflicting claims. While the conclusion of \zhang is correct, the counter-example they present fails to identify systemic vulnerabilities of \pol. {Our analysis shows that} they target a weakened instantiation of \pol. This is evidenced by the fact that their results are either not reproducible or easily thwarted by adjusting hyperparameters of the verification mechanism{, as we show in Section~\ref{ssec:limitations_prior_spoofing}}. Instead, 
{our proposed attacks are \textit{always} guaranteed to generate a spoof for the proposed \pol verification mechanism.} Further, this can be done for a \textit{fraction of the cost} needed by past spoofing strategies. We {arrive at} this understanding of {the} fundamental vulnerabilities of PoL by systematically studying the two roles \pol verification plays: (a) efficiently verifying that a sequence of model updates contained in a proof are valid, and (b) establishing precedence by {preventing spoofing}.

\looseness=-1
The first role, {that of} proof verification, reduces verification to a detection problem {\ie detecting valid training trajectories}. Indeed, multiple sources of noise make it impossible for the verifier to exactly reproduce each step of gradient descent contained in the  prover's proof. Avoiding noise in learning is not always possible{;} it is sometimes {even} desirable. Accelerators like {graphics processing units (GPUs)} introduce noise in training but {they} are needed to scale learning to large datasets. This means that some tolerance to noise needs to be built into the verification process. {This} creates a sensitivity-specificity trade-off {such that} {the verifier must tune the noise threshold to simultaneously maximize the acceptance of legitimate steps and rejection of invalid steps.} {However,} this tuning is currently not possible because we lack a precise model of how this noise impacts training, and consequently the trajectory taken by the learner. {Thus,} we find that we cannot formulate a proof verification mechanism that is optimal, \ie provably minimizes an adversary's capability to disguise their trajectory within the noise inherent to training.

These difficulties are {exacerbated} by {previously proposed} approximations to the verification mechanism{, which aim to} improve its computational efficiency. For instance, \jia propose to verify only a subset of {the} updates in a proof. {However, as we demonstrate, this approach} opens {an} attack surface for {adversaries} to force the verification mechanism to verify a subset of updates of their choice. {This leads us to formalize the assumption implied by \jia, and identify the gap in theory required to develop a provably robust and efficient verification mechanism.}

The second role of a proof (and verification) provides a different yet {equally} fundamental explanation for {the difficulty of proving} robustness. Crafting a valid proof, once given access to a trained model, is assumed to be at least as difficult as generating the proof {naturally} as the model is being trained. However, {formally} proving this reduces to {\em open problems} in learning theory. In the presence of a non-convex learning objective, as is the case in deep learning, multiple models corresponding to different local minima can be returned by {SGD}. {Then, to obtain a formal guarantee on the computation (to spoof a valid proof) in this setting, we must know if knowledge of one such minimum guarantees a unique trajectory.} This is an open problem and would have implications beyond the \pol protocol, for instance towards knowledge transfer in machine learning. 

{Based on} this analysis, we conclude that two classes of solutions are needed to address these limitations of the current verification mechanism. First, we need better models of noise dynamics in learning. For instance, we show how capturing the direction alongside the magnitude of individual gradient steps helps refine such models of noise. For the second, we will need methods to guarantee that having access to the final weights {does} not enable spoof generation at lower {computational cost} than training. We identify new \pol protocols that {highlight} commitment mechanisms as a promising direction. We expect that particular care will need to be given to committing to a dataset in order to establish precedence. 
To summarize, our contributions are:
\begin{itemize}
\item We systematically study the assumptions needed for \pol's robustness based on the two roles verification plays: (a) efficient verification with tolerance to noise; and (b) establishing precedence.
\item We highlight the need for better models of noise in learning to instantiate an optimal verification mechanism {and} illustrate one potential avenue for doing so. {However,} obtaining optimal verification mechanisms remains an open problem. As a consequence, we show in {Section}~\ref{sec:fundamentals_revisited} that our analysis of noise tolerance can be exploited to mount {the first practical attacks against \pol that are {both} reproducible across different PoL {configurations and} {are} significantly more efficient than prior work}. The code is provided at 
\url{https://github.com/cleverhans-lab/practical-attacks-against-pol}
.
\item However, theoretical limitations in establishing precedence with \pol do not yet result in practical attacks (see {Section}~\ref{ssec:stochastic_spoofing_experiment}). {We formulate reductions of the precedence problem as a first step in understanding the possibility of robustness for this assumption (see {Section}~\ref{sec:theory_cheap_spoofing}).}
\end{itemize}

\section{Background}
\label{sec:background}

{\pol} relies on an asymmetry in the training protocol arising from the highly complex and non-linear nature of training {DNNs}~\cite{saxe2013exact}. {The authors of \pol draw connections with proof-of-work~\cite{dworknaor1992, preneel_proofs_1999}, as they demonstrate how gradient inversion (see Section VII-B in Jia~\etal~\cite{jia2021proof}) is at least as expensive as gradient computation.} Thus, the authors hypothesize that gradient calculations on data play a role similar to one-way functions~\cite{diffie1976new}. {We} will revisit this {hypothesis} in our manuscript. Similar to proof-of-work, \pol should{, }ideally{,} prevent an entity from claiming they have trained a model without having spent at least {a} comparable {amount of} computational effort. 

In the rest of the {paper} we use lower-case, bold-faced notation to capture random variables. See Table~\ref{tab:notation_attack} {in Appendix~\ref{app:notations}} for {a list of} commonly used {notations}.

\subsection{Primer on \pol}
\label{subsec:primer}

The framework~\cite{jia2021proof} assumes the prover $\prover$ honestly trains a machine learning model in $T$ steps to obtain parameters $W_T$. {\pol is defined as follows.} 

\begin{definition}\label{definition:prf}
For a prover $\prover$, a proof is denoted as $\pw[\prover, {W_T}] = (\pWeights, \pBIndex, \pSign, \mathbb{A})$ where all elements of the tuple are ordered sets indexed by the training step $t \in [T]$. In particular, (a) $\pWeights$ is a set of model-specific information obtained during training{;} (b) $\pBIndex$ denotes information about the specific data points used to obtain each state in $\pWeights${;} (c) $\pSign$ represents cryptographic signatures of the training data{;} and (d) $\mathbb{A}$ is auxiliary information that may or may not be available to an adversary $\adv$, such as hyperparameters $\pMeta$, model architecture, optimizer, and loss choices. 
\end{definition}

If $\prover$ logs information for every step $t$, the exact training process (culminating at $W_T$) should ideally be reproducible. The memory footprint of {$(\pBIndex, \pSign, \mathbb{A})$} is often small. However, {storing the model weights (\ie a part of $\pWeights$)} incurs high overhead. Thus, it is common practice to log weights periodically at every $k^{th}$ step; {$k$ is known as the checkpointing interval.}

Motivations for \pol include substantiating {ownership claims for a} specific {set of} weight{s} $W_T$ or verifying the correctness of delegated computations. The latter may arise in the context of distributed learning~\cite{dean2012large}. The former is motivated by the threat of model stealing~\cite{tramer2016stealing}: the \pol protocol increases the cost of an adversary as it is now required to generate a proof for the model it has stolen {or obtained through insider access.}

\vspace{2mm}
\noindent{\bf Verification:} Without loss of generality, the verifier begins with $W_t$, and {utilizes} the information in $(\pBIndex, \pSign, \mathbb{A})$ to perform $k$ steps of training to achieve $W'_{t+k}$; the difference between two such weights is termed an {\em update}. The {newly obtained $W'_{t+k}$} is compared against the next stored weight $W_{t+k}$. In general, $W'_{t+k} \neq W_{t+k}$ due to entropy arising from low-level components, \eg low-level libraries and hardware~\cite{jagielski2019high, pham2020problems, Zhuang2021randomness}. We refer to such differences as noise or \emph{stochasticity} in training. This {noise} limits the maximum {value of} $k$ that can be chosen. $W'_{t+k}$ is deemed valid if $d(W_{t+k}, W'_{t+k}) < \delta$, \ie the recreated weight $W'_{t+k}$ is within a $\delta$ error threshold of the prover-generated weight $W_{t+k}$ using some distance function $d$ (typically an $\ell_p$ norm). The threshold $\delta$ is {set} by the verifier prior to proof verification and is tuned such that $\varepsilon_{repr}(t) < \delta \ll $ \rd, where $\varepsilon_{repr}(t)$ captures the reproduction error due to low-level randomness (at various training steps) and \rd is a reference distance estimated by re-running the training protocol {with varying sources of stochasticity} and recording the deviation upon completion. \jia also {defines} normalized reproduction error ($||\varepsilon_{repr}||$) as $\max_t \varepsilon_{repr}(t)/$\rd.

As {the notation} will become important when considering spoofing attacks {in Section~\ref{sec:theory_cheap_spoofing}}, we formalize the set of \textit{valid proofs} ending in a particular weight $W_T$ {similar to what was done by} Thudi \etal~\cite{thudi2021necessity}: Let $A_{D,W_T}$ be the set of $(\mathbf{g},d,\boldsymbol{\delta})$-proofs ending in $W_T$ generated by a specific dataset $D$, \ie those obtained using update rules $\mathbf{g}_i \in \mathbf{g}$ passing thresholds $\boldsymbol{\delta}$ {(\ie $||\varepsilon_{repr}||<\delta$)} in metric $d$. With this formalization, we can associate to {\em honest training} a distribution on the set of valid proofs ending in $W_T$.

\begin{definition}[Honest Training]
Honest training is given by a probability measure $\mu$ on the event space $A_{D,W_T}$.
\end{definition}

\noindent{\bf Efficient Verification:} To speed up verification, it is possible to use heuristics to select {\em which} specific updates to verify. {However, this} introduces a trade-off between computation savings and verification accuracy. \jia utilize the {\em top-$Q$} mechanism for probabilistic verification. {The} verifier selects the $Q$ largest
(in their $\ell_p$ norm) updates of each epoch for verification; the intuition for this is that larger updates exist primarily in falsified proofs.

\subsection{Creating Spoofs} 
\label{subsec:spoofs}

\jia define a spoof as any proof that {passes verification and} requires {\em lesser computation} to obtain than honest training {would} {(see Table~\ref{tab:spoofing} in Appendix~\ref{app:notations} for the full list of spoof categories they introduce)}. One of the key contributions of this paper is to analyze why, {if at all}, such spoofs exist. For now, we briefly review spoofing schemes known in the literature.

\vspace{1mm}
\noindent{\bf 1. Sequence Inversion~\cite{jia2021proof}:}
We assume an adversary has access to $W_T$ and the data used to train the model, but not to $\pBIndex$ {or} $ \mathbb{A}$. {Here, an adversary} aims to invert gradient descent: given $W_T$, find a corresponding $W_{T-1}$ that was used to obtain it. The authors show that such a process is difficult due to {the} increasing entropy {as training progresses} and {that it} is computationally lower bounded by the cost of honest training.

\vspace{1mm}
\noindent {\bf 2. Directed Retraining~\cite{jia2021proof}:} Under the same assumptions as {(1), an adversary employing this strategy aims to} either {create} (a) structurally correct proofs {\ie spoofs that pass verification by generating invalid updates} or (b) {shorter ones by} artificially directing the weights to $W_T$ quicker. The authors showed that unless the entire proof {was} valid, any discontinuities produced by methods in (a), \eg concatenating proofs, would be detected by a verification mechanism that {checked} the {few} largest updates first. They {also} argued that approaches in (b) require custom training algorithms (\eg regularizers, loss functions, etc.) with direct knowledge of the desired final weights $W_T$, and thus {would} fail verification.

\vspace{1mm}
\noindent{\bf 3. Adversarial Examples for \pol:} \zhang introduce two techniques to generate shorter {\em structurally correct} spoofs (see {Section IV-A} in \jia), based on evasion~\cite{goodfellow2014explaining, szegedy2013intriguing} (\ie perturbing inputs to create adversarial examples which result in erroneous model predictions). These strategies assume the adversary has access to $W_T$ and the data that was used to obtain the proof, but no other information. Herein, we define $\ustep_k(W_i, X_i)$ to represent the update to go from $W_i$ to $W_{i+k}$ with data $X_i$. We utilize $\hat{W}$ to denote a weight created by the adversary.

\vspace{1mm}
\noindent{\bf 3.1. Synthetic Adversarial Update:} The objective of the adversary is to create synthetic data $\hat{X}$ such that{, w.l.o.g,} the following is possible: $\hat{W}_{t+n} = \ustep_n(\hat{W}_t, \hat{X})$. {The synthetic data $\hat{X}$} is termed an ``adversarial example'' by the authors. Since generating these adversarial examples requires additional computation, the adversary {only performs} $\hat{T}-1$ steps of legitimate training with $\hat{T} \ll T $ {leaving enough computation for the single step from $\hat{W}_{\hat{T}-1}$ to  $W_T$ so that the overall computation is still lesser than honest training.} {This is their ``Attack 1''.} However, the authors only conceptually describe the attack and do not evaluate it because they find it is difficult for the optimization to converge. Despite our best efforts, we also failed to have our implementation of this strategy converge and synthesize data that satisfies the adversary's objective. Thus, given that \zhang do not provide evidence that this attack strategy can succeed for the adversary, {we will later provide an intuition as to why ``Attack 1'' is unlikely to succeed, but }we do not consider it further in our work.

\vspace{1mm}
\noindent{\bf 3.2. Synthetic Checkpoint Initialization:} \zhang {also} propose a technique to choose intermediary weights that exploit the threshold $\delta$ picked by the verifier. Recall that this threshold allows the verifier to tolerate noise induced by stochasticity in SGD. The adversary can choose pairs of weights $\hat{W}_{t}$ and $\hat{W}_{t+k}$ such that $d(\hat{W}_{t}, \hat{W}_{t+k}) \ll \delta$, $\forall t, k${, for example, by} linearly interpolating between a chosen $\hat{W}_0$ and $W_T$. The authors also propose techniques to minimize the distance $d(\hat{W}_0, W_T)$ between the model initialization and final weights which {naturally} further minimizes {each intermediate update distance $d(\hat{W}_{t}, \hat{W}_{t+k})$}. {To achieve this,} data $\hat{X}$ is initialized with training data and perturbed such that $d(\hat{W}_t, \ustep_k(\hat{W}_t, \hat{X})) \approx 0 \ll \delta$ {(note that, semantically, the data is synthesized to obtain ``small'' updates instead of the ``correct'' updates as in Attack 1)}; this results in $d(\hat{W}_{t+k}, \ustep_k(\hat{W}_t, \hat{X})) \ll \delta$ as $d(\hat{W}_{t}, \hat{W}_{t+k}) \ll \delta$ and thus passes verification. {The authors} call this ``Attack 2''.\footnote{The authors also introduced ``Attack 3'', which is a more computationally efficient, yet conceptually similar implementation of Attack 2. Given that the two achieve the same performance, we exclusively consider {Attack} 2 in {Section}~\ref{ssec:on_finding_cheap_pols}.}

\subsection{Limitations of Prior Spoofing Strategies}
\label{ssec:limitations_prior_spoofing}

Both \zhang and {our team} failed to implement an instance of Attack 1 that works. Here, we provide a theoretical intuition {as to} why. Attack 1 is essentially trying to perturb a low-dimensional input to modify a much higher-dimensional output {(\ie crafting an adversarial example that can output an arbitrary gradient update)}. {For example, a CIFAR-10 data point is a $32\times32\times3=3072$-dimensional vector but a gradient update for a ResNet-20 is a $270,000$-dimensional vector.} The {(a) significant difference in dimensions, and (b) fact that the input is discrete (only integers from 0 to 255), whereas the update is continuous,} means that the possible outputs {from Attack 1} is a tiny subset of all possible model updates one could need to replicate. {This makes Attack 1 unlikely to succeed.} 

{On the other hand, Attack 2 is more practical as the goal is only to reduce the magnitude of the high-dimensional 
{gradient.} However, {it is still difficult to achieve gradients with near zero magnitude} 
due to {how deep learning algorithms are implemented}. For example, even in the case that the model's prediction on a data point perfectly matches its label, values added to certain layers of the model (\eg softmax) for numerical stability would cause the loss value{, and thus the gradient update,} to be non-zero. Therefore, Attack 2 cannot work when $\delta$ is small, which is the case when $k$ is small {(see Section VI-C in \jia)}. {However, } \zhang made an assumption that the attacker is able to set the value of $k$ to obtain $\delta$ that is large enough. For instance, they did so in their experiments on CIFAR-100 and set $k=100, \delta=0.1$ ({Note}: this means that after training for 100 steps, the model can be $10\%$ different). {T}his is an incorrect assumption: $k$ should be set by the verifier and is thus out of the attacker's control. In the case of CIFAR-100, we empirically found that Attack 2 would not converge if $k$ and $\delta$ are reduced by an order of {magnitude}. In other words, Attack 2 only works for certain settings of \pol verification and the verifier can easily prevent it by using a small $k$.}

{Lastly, both Attack 1 and Attack 2 propose {using} gradient-based methods to solve non-convex optimization problems whose objective functions already contain a derivative. This means a second-order derivative has to be computed for every single step, which is unnecessarily expensive ({\ie the same spoofing effect can be achieved without the second-order derivative, see} {Section}~\ref{ssec:good_training}). {In addition, while \zhang only focused on exploiting the non-optimal noise thresholding mechanism, they didn't analyze its fundamental causes and other sources of vulnerabilities. }{In contrast, we will systematically explore the vulnerabilities of \pol and propose new attacks that stem from them. The proposed attacks are {computationally affordable} and do not over-assume the adversary's power. Furthermore they are also generally applicable to different \pol configurations.}}

\section{Threat Model}
\label{subsec:threat_model}

In the rest of the paper, we consider the following threat model unless otherwise specified{:}

\begin{enumerate}
\item An adversary $\adv$ has complete knowledge of $W_T$ and the architecture used to obtain it.
\item $\adv$ has knowledge {of the verification parameters $\delta$ and $Q$,} and the {selection mechanism for which updates will be verified.}
\item $\adv$ does {\em not} have access to any other information from training: e.g., intermediate training steps, or the sources of randomness used to obtain $W_T$. 
\item $\adv$ has access to the training dataset (or distribution). {Note this is assumed by previous works~\cite{jia2021proof,zhang2021adversarial}, and so we follow it when studying the robustness of \pol. However, our proposed attacks do not rely on this because the verifier is capable of keeping the dataset secret (refer Section~\ref{sssec:data_commitment}).}
\end{enumerate}

Finally, our primary analysis is for DNNs with non-convex loss landscapes as they are more computationally expensive to train and {are more likely to be targeted by model stealing attacks.}

\section{Efficient Verification of Valid Proofs}
\label{sec:giving_guarantees_with_pol}

Recall that one of the two key roles played by the proof verification mechanism is efficiently verifying whether a sequence of model updates can be obtained by SGD. This is done by reproducing these updates, and if they pass, they are termed ``valid.'' To understand the robustness of \pol verification, we theoretically analyze the role it plays from two perspectives: (a) {\em correctness} (Section G.1 in \jia), meaning that the honest model trainers' proofs should be correctly validated (even on a different hardware/software stack); and (b) {\em verification efficiency} (Section G.3 in \jia). We first introduce the necessary conditions for these desiderata to be satisfied, namely {\em reproducibility and representativeness}. We will formalize these assumptions later in Section~\ref{ssec:good_training} and Section~\ref{sssec:topq_theory} respectively:

\begin{enumerate}
\item {\bf Reproducibility}: Individual gradients are reproducible up to a small error $\delta \ll 1$ if the per-step training data and metadata required to obtain them are logged and controlled for.
\item {\bf Representativeness:} The validity of a sequence of model updates can be implied by the validity of a smaller subset. Thus, verifying the smaller subset is equivalent to verifying the larger one in terms of security guarantees.
\end{enumerate}

\subsection{On Correctness of Step-wise Verification}
\label{ssec:good_training}

\begin{tcolorbox}[after={\refstepcounter{footnote}\footnotetext{In Section~\ref{subsec:spoofs} we introduced the notation $\ustep_k(W_i, X_i)$ to represent model updates with an implicit assumption that metadata $M$ is utilized. Here, we overload the notation to explicitly include the metadata.}}]
\begin{assumption}[Reproducibility]
\label{assum:reproduce}
For any sequence of intermediate model states obtained from honest training for $T$ steps, $\{ W_1, \cdots, W_T \}$, if training step $t$ is reproduced (\eg by the verifier) using the same update rule $\mathcal{U}_1$, training data $D_t$, and metadata $M_t$ to produce $\widehat{W}_{t+1} = \mathcal{U}_1(W_t,D_t, M_t)$.\footnotemark[\the\numexpr\value{footnote}+1] Then there exists a bound, $\delta$, for this reproduction error such that $||W_{t+1}-\widehat{W}_{t+1}|| \leq \delta$ $\forall t \in [T-1]$.
\end{assumption}
\end{tcolorbox}

\begin{figure}[t]
\centering\hspace{-5mm}
\includegraphics[height=4cm]{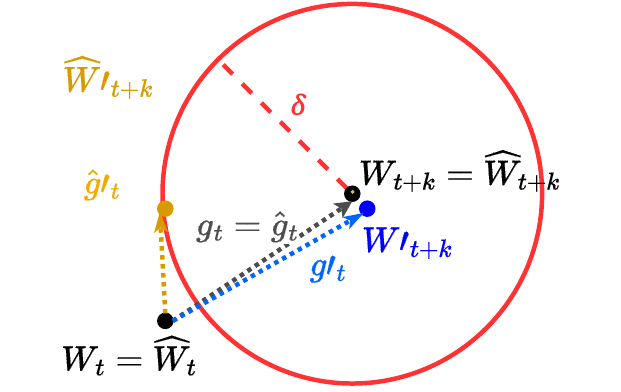}
\caption{ \textbf{Attacks targeting loose noise thresholds $\delta$.} For a legitimate step, the discrepancy between $g_t$ (update from the proof, in black) and $g'_t$ (verifier-reproduced update, in blue) should be solely due to the noise from hardware/software, and should lie within the $\delta$ ball. However, for loose thresholds, the adversary may be able to create adversarial metadata such that the verifier-reproduced update $\hat{g}'_t$ (in yellow) points in a different direction than the one in the proof ($\hat{g}_t$, in black) and lies within the $\delta$ ball. The verifier will incorrectly accept this step despite it being ``different" from $\hat{g}_t$.}
\label{fig:delta_not_tight_illustration}
\end{figure}

A fundamental issue facing \pol verification is how the noise in training manifests itself, which translates to what degree the training can be reproduced. Ideally, this is to be captured by the verification threshold $\delta$ (see Assumption~\ref{assum:reproduce}). The original work~\cite{jia2021proof} proposed a step-wise verification method based on the (over)simplified assumption that $||W_{t+1}-\mathcal{U}_1(W_t,D_t,M_t)||$ is i.i.d for all $t$. Thus, $\delta$ was assumed to have a fixed value. However, if the $\delta$ threshold is not tight, then the adversary may be able to exploit it to generate spoofs. One way of doing so is by designing data and/or metadata that makes the verifier's reproduced update point in a new direction (though still within the $\delta$-ball) compared with the original update in the proof. This spoof contains a new trajectory to a new final weight $W'_t$, as shown in Figure~\ref{fig:delta_not_tight_illustration}. In fact, Attack 2 by \zhang implicitly exploits this vulnerability. Thus, we seek a complete (\ie necessary and sufficient) characterization of the distribution(s) of this noise throughout training, so as to enable the selection of the optimal (tightest) thresholds. This will enable a verifier to guarantee that they accept valid training steps while minimizing the threat surface.

To this end, let us assume the verifier produces \texttt{ACCEPT} or \texttt{REJECT} decisions for a training step. We can quantify the decision process by the True Positive Rate, $TPR=\frac{TP}{(TP + FN)}$ where $TP$ is the number of accepted valid gradients and $FP$ is the number of rejected but valid gradients. Therefore, the question becomes: ``How large is the minimum threshold to retain some fixed TPR?'' We formalize these requirements on the verifier in Definition~\ref{def:set-verifier}. 

\begin{definition}[$(d, \tau)$-verification strategies]
\label{def:set-verifier}
For proofs of length $T$, $V_{d, \tau}$ is defined as the set of all verification strategies which use per-step thresholds $\{ \delta_1 {(W_1,M_1)}, \cdots, \delta_T {(W_T,M_T)}\}$ (possibly depending on the intermediate checkpoints $W_t$ and corresponding training metadata $M_t$ such as hyperparameters) and produces verification decisions \texttt{ACCEPT}, \texttt{REJECT} with a required per-step TPR of at least $\tau$ for any weight, based on a metric $d$ over the training steps.
\end{definition}

Our question is now reduced to the following: ``Does an optimal $(d,\tau)$-verifier exist?'' Note that here optimality is defined as having the {\em smallest per-step threshold} (for each step) that maintains the TPR rate of $\tau$.

\subsubsection{Existence of Optimal Strategies}
\label{thudi_a}

To take the first step toward proving the existence of the optimal step-wise verification strategy for a given metric $d$, we begin by proving its existence for per-step $\ell_2$ norm verification in Lemma~\ref{lem:min_threshold}. We then generalize this result to an entire class of metrics that satisfy the requirement that the boundary of their metric balls has Lebesgue measure $0$. Importantly, this includes all $\ell_p$ metrics and cosine similarity, which are commonly used in machine learning. This is obtained for free in Corollary~\ref{cor:all_metrics} from the proof of Lemma~\ref{lem:min_threshold}.

\begin{lemma}[Minimum Threshold]
\label{lem:min_threshold}
Let $\mathbf{g}_i \in \mathbf{W}$ be the random variable for the $i^{th}$ update from a given checkpoint $W_i$, where $\mathbf{W} = \mathbb{R}^n$ is the weight space. Assume it is absolutely continuous~\cite{folland1999real} with respect to the Lebesgue measure. Then, for a given TPR $\tau$, there exists a minimum $\ell_2$ threshold $\delta$ centered at the mean of $\mathbf{g}_i$ s.t. the TPR for $\delta$ is $\tau$, and for any $\delta' < \delta$, the TPR is less than $\tau$.
\end{lemma}

\begin{corollary}[All Metrics]
\label{cor:all_metrics}
In the same setup as Lemma~\ref{lem:min_threshold}, if instead of the $\ell_2$ metric, some other metric $d$ is used such that the boundary of the metric balls of $d$ have $0$ Lebesgue measure, then there exists a minimum threshold $\delta$ for $d$ centered at the mean of $\mathbf{g}_i$, such that for any $\delta' < \delta$, the TPR $< \tau$.
\end{corollary}

The idea is that one can interpolate between different thresholds to find the thresholds with the desired TPR. However, to do so we need to check continuity properties.

\begin{proof}
First we will use {\em continuity from above} (for measures) to show there is some $\delta$ s.t. the TPR $< \tau$. Note: as we are dealing with probability measures, \ie all sets have finite measures, we need not worry about the finiteness requirement for continuity from above. Recall that the measure is a continuous function from a connected domain, so by intermediate value theorem, there is some $\delta$ s.t. the TPR of $\tau$. Moreover, the pre-image of thresholds that give a TPR of $\tau$ is a closed set (by definition of being a continuous function). Thus, by taking the subset within $[0,\delta]$, we have a compact set. Hence, by extreme value theorem, there exists a minimum $\delta$ with TPR of $\tau$. 

Let $\overline{g}_i = \mathbb{E}(\mathbf{g}_i)$, and $B_{d,2^{n}}(\overline{g}_i)$ be the balls centered at $\overline{g}_i$ of radius $2^n$ in metric $d$ for $n \in \mathbb{Z}$. For some large $N$, note that $\cap_{n \in \mathbb{Z}, n \leq N} B_{d,2^{n}}(\overline{g}_i) = \overline{g}_i$ and so we have $\mu( \cap_{n \in \mathbb{Z},~n \leq N} B_{d,2^{n}}(\overline{g}_i)) = 0$ (a single point set and hence measure $0$). By absolute continuity, we then have $\mu_{\mathbf{g}_i}(\cap_{n \in \mathbb{Z}, n\leq N} B_{d,2^{n}}(\overline{g}_i)) = 0$. By continuity from above $\lim_{n \rightarrow -\infty} \mu_{\mathbf{g}_i}(B_{d,2^{n}}(\overline{g}_i)) = 0$. So, to conclude, we have $\forall \tau > 0$ there exists some $n$ s.t. $\mu_{\mathbf{g}_i}(B_{d,2^{n}}(\overline{g}_i)) < \tau$. 

By continuity from above and below (for measures) taken with respect to balls $B_{d,r}(\overline{g}_i)$ for $\mu_{\mathbf{g}_i}$, and noting that $\mu_{\mathbf{g}_i}(\overline{B_{d,r}}(\overline{g}_i)) = \mu_{\mathbf{g}_i}(B_{d,r}(\overline{g}_i))$ where $\overline{B_{d,r}}(\overline{g}_i)$ is the closure of the ball (which has measure $0$), we see that $\mu_{\mathbf{g}_i}(B_{d,r}(\overline{g}_i))$ is a continuous function from $[0,\infty] \rightarrow [0,\infty]$ (where the variable is the radius $r$ of the ball). Since it is a continuous function from a connected domain, there is some $\delta$ such that the TPR is $\tau$ (by intermediate value theorem). Note that $\tau \leq 1$, and we know the measure of the whole space is equal to 1; this gives the upper-bound for intermediate value theorem. Moreover, the pre-image of thresholds that give TPR $=\tau$ is a closed (and non-empty) set; taking the subset within $[0,\delta]$ we have a compact set, and hence by the extreme value theorem, there exists a minimum $\delta$ with TPR equal to $\tau$. This concludes the proof. 
\end{proof}

\noindent{\bf Discussion:} What Lemma~\ref{lem:min_threshold} and more generally Corollary~\ref{cor:all_metrics} shows is that for those metrics whose Lebesgue measure of the boundary of their metric balls is $0$, there \emph{exists} an optimal per-step verification threshold to obtain the desired TPR. The question now becomes: ``How do we instantiate the optimal step-wise verification strategies?'' As the first observation in this direction, we empirically demonstrate any constant threshold (\ie staying the same over all iterations and weights) can be loose in certain settings allowing spoofing (in Section~\ref{ssec:on_finding_cheap_pols}). So the question finally becomes: ``How could we leverage the hyper-parameters and weights to devise tighter thresholds?''

\subsubsection{On FPR Considerations}

To complement the discussion on optimal thresholds given constraints on TPR, we can similarly consider how constraints on the false positive rate (FPR) translate to restrictions on the verification threshold. We proceed as follows: 
we make an assumption about the adversaries that associates an absolutely continuous distribution of ``false,"
\ie spoofed gradients to each checkpoint $W_i$. Let us denote the random variable of these spoofed gradients as $\mathbf{\hat{g}}_i$, and the measure as $\mu_{\mathbf{\hat{g}}_i}$. Note that the FPR of a given verification mechanism is then $\mu_{\mathbf{\hat{g}}_i}(E)$, where $E$ is the set of accepted updates. 
Recall from Section~\ref{thudi_a}, 
we have $E = B_{d,\delta}(\bar{g_i})$, \ie the ball of radius $\delta$ centered at the mean honest update $\bar{g_i}$ using metric $d$ (satisfying the assumption of Corollary~\ref{cor:all_metrics}).

In practice, we want to upper-bound the FPR rate. That is, we want to pick a threshold $\delta$ that does not accept too many of the spoofed gradients, \ie $\mu_{\mathbf{\hat{g}}_i}(B_{d,\delta}(\bar{g_i})) \leq \lambda$. In this case, we can prove that, for a given $\lambda$, there is a maximum threshold $\delta$ such that $\mu_{\mathbf{\hat{g}}_i}(B_{d,\delta}(\bar{g_i})) \leq \lambda$.

\begin{lemma}[Maximum Threshold]
\label{lem:max_threshold}
Suppose metric $d$ satistfies the assumption of Corollary~\ref{cor:all_metrics}. Then, given $0 \leq \lambda <1$, there exists a maximum threshold $\delta$ for $d$ centered at the mean of the honest gradients $\mathbf{g}_i$ s.t. for any $\delta' > \delta$ the FPR $> \lambda$.
\end{lemma}

\begin{proof}
The approach is completely analogous to Lemma~\ref{lem:min_threshold} and Corollary~\ref{cor:all_metrics}. Inspecting the proof of the existence of a $\delta$ s.t. the ball has a desired measure, which only relied on having an absolutely continuous measure and the condition on the metric $d$, we immediately find that there exists at least one $\delta$ s.t. $\mu_{\mathbf{\hat{g}}_i}(B_{d,\delta}(\bar{g_i})) = \lambda$. 
Similarly, the proof of Lemma~\ref{lem:min_threshold} also showed the set of all thresholds giving the specified measure is closed, so we conclude the set of $\delta$ s.t. $\mu_{\mathbf{\hat{g}}_i}(B_{d,\delta}(\bar{g_i})) = \lambda$ is closed.

Note that this set is necessarily bounded for $\lambda <1$, as we have $\lim_{\delta \rightarrow \infty} \mu_{\mathbf{\hat{g}}_i}(B_{d,\delta}(\bar{g_i})) =1$, and so in particular there exists $\delta'$ s.t. for any $\delta > \delta'$ we have $\mu_{\mathbf{\hat{g}}_i}(B_{d,\delta}(\bar{g_i})) > \lambda$.
{Therefore, the set of thresholds that exactly give FPR $\lambda$ is non-empty, closed, and bounded. In particular, we have by the extreme value theorem that a maximum threshold $\delta$ exists. This concludes the proof.}
\end{proof}

\looseness=-1
In light of this result on thresholding when considering FPR, and the previous results on thresholding when considering TPR, one might ask what happens when we design verification methods with constraints on both. The main point is that if the upper-bound of Lemma~\ref{lem:max_threshold} is below the lower-bound of Lemma~\ref{lem:min_threshold}, then we cannot achieve both high TPR and low FPR. If, in fact, the upper-bound is higher than the lower-bound, then we can follow the lower-bound and achieve both high TPR and low FPR. In the following discussion we will focus on obtaining the optimal threshold for TPR with the understanding that this is sufficient (when high TPR and low FPR are both possible).

\subsubsection{On Constructing Optimal Step-wise Verification Strategies}
\label{thudi_b}

Constructing the optimal step-wise verification strategy, and especially proving its optimality, remains an {\em open problem} because the noise is always non-zero and dependent on the software and hardware used by the prover. However, notice that the training is successful despite the presence of this noise. This suggests that although optimality cannot be proven, it is possible to construct a characterization of noise that is compatible with the optimization objective of training. This characterization can then inform the design of thresholds for \pol verification mechanisms by serving as guidelines for the type of noise we should expect when creating and verifying proofs.

\begin{figure}[t]
\centering\hspace{-5mm}
\includegraphics[height=3cm]{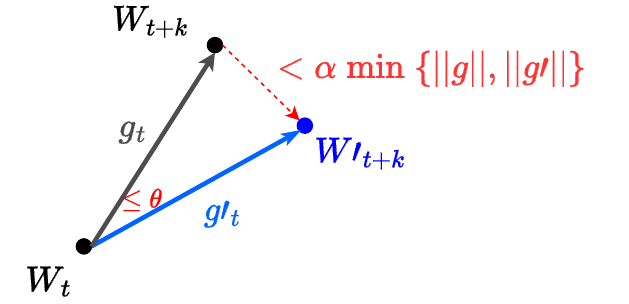}
\caption{ \textbf{Illustration for Lemma~\ref{lem:adaptive_threshold}.} Given the update from the proof, $g$, and verifier-reproduced update, $g'$, there exists some $\alpha$ such that the angle between $g$ and $g'$ is bounded by $\theta$ if $||g - g'|| < \alpha \cdot \min\{||g||,||g'||\}$.}
\label{fig:tight_delta_cone}
\end{figure}

Intuitively, noise altering the update's direction has the greatest potential to impact the convergence of a training run---this places a high (probability) prior for valid updates to be in certain directions. Assume that if the reproduced update pointed in the opposite direction of the original update. This would indicate that the reproduced update was actually the solution to the opposite objective (e.g., maximization instead of minimization). Thus, with high probability, the noise should have lower variance in the direction of the logged updates. However, this is not captured by the $\ell_2$ metric, which only compares the magnitudes. Thus, it would instead be better to inform the design of the verification mechanism with the necessary condition that noise must have lower variance {\em in the direction of training}. To this end, we can prove there exists a verification scheme that bounds both the error in the direction and magnitude of the  $\ell_2$ metric. This is given in Lemma~\ref{lem:adaptive_threshold}, and an illustration is shown in Figure~\ref{fig:tight_delta_cone}.

\begin{lemma}
\label{lem:adaptive_threshold}
$\forall \theta \in [0,2\pi)$, $\exists \alpha$ s.t. angle between $g$ and $g'$ is $\leq \theta$ if $||g - g'|| < \alpha \min\{||g||,||g'||\}$.
\end{lemma}
\begin{proof}
Let $\theta$ denote a desired bound on the difference in angles between $g$ and $g'$; this bound defines a convex conic set $S$ in $\mathbb{R}^n$, and at $\frac{g}{||g||}$ one can fit a ball with radius $\alpha$ in this set. By convexity, and by the origin being in the convex cone, it follows that by interpolating along the line to $0$, the $B_{\ell_2,s \cdot \alpha}(\frac{s \cdot g}{||g||})$ ball for all $s \geq 0$ is in $S$. Thus, by taking $s = ||g||$, if $||g' - g|| \leq \alpha \cdot ||g|| $ then $g' \in B_{\ell_2,s \cdot \alpha}(\frac{s \cdot g}{||g||}) \subset S$. That is, the angle between $g'$ and $g$ is less than $\theta$. 
\end{proof}

Lemma~\ref{lem:adaptive_threshold} informs how we can design step-wise verification mechanisms that can bound both the difference in angle and $\ell_2$ norm. To do so, one may set the verification threshold $\delta$ to be $\alpha \cdot \min\{||g||,||g'||\}$ instead of a constant. This analysis suggests a first step towards an improved step-wise verification strategy, and we show in Appendix~\ref{app:additional_figures} ({Figures}~\ref{fig:defense_intuition_adp_delta} and~\ref{fig:adaptive_delta_all}) how using $\alpha =1$ leads to an improved thresholding scheme that successfully thwarts all current and past attacks. However, the resulting verification mechanism is still not provably optimal. Indeed, just like the threshold on the $\ell_2$ norm of the difference, the threshold on the angle would also be impacted by implementation details like on what machine the DNN is trained, what library is used, etc. The following open questions remain.

\begin{tcolorbox}
\bf Open Questions: (a) How do we instantiate the optimal verification strategy (if we can)?; and (b) Can an adversary still bypass the optimal verification strategy?
\end{tcolorbox}

\subsection{On Efficient Verification}
\label{sssec:topq_theory}

We now turn to the efficiency perspective of proof verification and its impact on our ability to formally reason about the robustness of verification. To make verification more efficient, the verifier may either (a) only verify a subset of training steps; or (b) spend less cost on verifying individual steps (refer Section~\ref{subsec:primer}). \jia propose selecting the $Q$ updates with largest magnitude to make verification more efficient through (a). 
However, recall that proof correctness is defined on a {\em complete} sequence of model updates contained in the proofs. Thus, verifying only a subset will not guarantee that only valid proofs are accepted
{\em unless that subset is representative (see Definition~\ref{def:subset})}%
. A simple counter-example is if an adversary is able to create a {\em valid} subset of model updates which is part of a larger, {\em invalid} sequence, such that the subset passes verification with some non-trivial probability. In this case, the larger {\em invalid} sequence also passes verification. This requirement for representative subsets of updates raises the question of if an adversary can somehow control which updates are verified so as to create a {\em valid} subset of model updates from a larger, {\em invalid} sequence.

To obtain a robust verification mechanism that is more efficient than verifying the complete sequence of model updates, we need to show the existence of a subset of updates such that its validity implies the validity of the entire training sequence. We define such a {\em representative subset} (in Definition~\ref{def:subset}) which satisfies the desired property 
(see Lemma~\ref{lemma:subset}). The total cost of representative subset selection together with verifying the subset should be less than that of verifying the entire proof sequence. Note that by this definition, the set of all updates in a proof is also a representative subset of itself. It is assumed that there exists a representative subset that is smaller than the entire proof (formally stated in Assumption~\ref{assum:represent}).

\begin{definition}[Representative Subset]
\label{def:subset}
A subset $\mathbb{S}$ of model updates is representative of a training process $\{g_1, \cdots, g_T\}$, if there exists any update $g_i$ such that $||g_i - g'_i|| > \delta_i$ (\ie the $i^{th}$ update is invalid), then there exists at least one $g_j \in \mathbb{S}$ with $||g_j - g'_j|| > \delta_j$, where $i$ could be equal to $j$.
\end{definition}

\begin{lemma}
\label{lemma:subset}
Verifying $\mathbb{S}$ is equivalent to verifying all the training updates.
\end{lemma}
\begin{proof}
When verifying $\mathbb{S}$, if all the updates pass, \ie $||g_j - g'_j|| < \delta_j$ $\forall g_j \in \mathbb{S}$, then by Definition~\ref{def:subset}, there does not exist any $g_i$ such that $||g_i - g'_i|| > \delta_i$ in the entire sequence of training updates. On the other hand, if at least one update in the representative subset did not pass, then the proof should be rejected as it captures another update in the entire sequence that will not pass verification.
\end{proof}

\begin{tcolorbox}
\begin{assumption}[Representativeness]
\label{assum:represent}
If a proof consists of $T$ training updates, then there exists a representative subset (see Definition~\ref{def:subset}) for it with a size less than $T$.
\end{assumption}
\end{tcolorbox}

\looseness=-1
Selecting a representative subset is, however, difficult. It is equivalent to finding model updates that (a) are necessary for achieving the final model state{;} or (b) would not exist if previous updates were not computed correctly. Researchers who study optimizers for DNNs have expended significant effort in studying similar questions, but it still remains an open problem. Nonetheless, we may leverage this definition to infer properties that an optimal selection mechanism (\ie one that is able to select the smallest representative subset) would possess. Based on Definition~\ref{def:subset}, the representative subset must contain individual updates that can be used to infer the validity of some other updates not in the subset. This means that the metric used to select such updates must be a function of multiple training updates and take into account the relationships between updates. Thus, we believe the selection mechanism for the top-$Q$ verification approach proposed by \jia does not satisfy Assumption~\ref{assum:represent}. It only considers the norm of individual training updates (in isolation). This results in concrete attacks---we instantiate one in Section~\ref{ssec:efficient_verification} showing how an adversary may create a subset of updates that seems to be valid, ergo forcing the verifier to perform verification on the valid updates but leaving the other invalid updates in the spoof untouched. Based on our findings, we frame the following open questions:

\begin{tcolorbox}
\bf Open Questions%
: (a) How to select a subset of training updates to represent a training process?; (b) What properties of the selected training updates make them more important than those not selected?; (c) How does this connect to optimization algorithms of DNNs?
\end{tcolorbox}

So far, we have provided a systematic assessment of the desiderata of proof verification to correctly and robustly play its role in verifying proofs. Violations of these desiderata may result in spoofs passing verification. In Section~\ref{sec:fundamentals_revisited}, we empirically validate our claims by introducing new attacks. Before this exposition, we first focus on understanding the power an adversary has for spoofing given our threat model of Section~\ref{subsec:threat_model}.

\section{On Understanding Cheap Spoofing}
\label{sec:theory_cheap_spoofing}

Up to this point, we have discussed the first role of the verification mechanism, namely, to efficiently verify that a sequence of model updates follows a valid gradient trajectory. However, adversaries may instead choose to target the second role of verification: establishing precedence. To do so, the adversary's aim is to find a valid proof at a cost lower than honestly training the model. Intuitively, this is possible since the adversary has knowledge of the parameters of the final model---knowledge that the honest trainer does not have.

This is where the second role of a proof---establishing precedence for the trained model---becomes important for characterizing \pol protocol robustness. We need to ensure that, given knowledge of the final weights $W_T$ obtained from honest training but not the rest of the proof of the original model owner, it is {\em impossible} to recreate any valid sequence (defined in Section~\ref{subsec:primer}) resulting in $W_T$ using a cheaper generation process than honest training. The efficacy of any proof generation process is measured through its ``cost,'' which is formally defined as follows:

\begin{definition}[Cost of a Proof]
We define cost as some function $C: A_{D,W_T} \rightarrow \mathbb{R}^{+}$, which represents the ``cost" associated with computing each proof. 
\label{def:cost}
\end{definition}

Proving impossibility of inexpensive proof creation naturally leads to the requirement for a {\em cheapness} assumption. We note that \jia require their proofs satisfy such a cheapness assumption in desideratum G.2 (Security). We state the assumption below: 

\vspace{2mm}

\begin{tcolorbox}
\begin{assumption}[Cheapness]
\label{assum:cheapness}
{Given a cost function $C$, for any algorithm $F: (D,W_T) \rightarrow A_{D,W_T}$, $\mathbb{E}[C(F(D,W_T))] \geq \mathbb{E}_{\mathcal{P} \sim A_{D,W_T}}[C(\mathcal{P})]$}
\end{assumption}
\end{tcolorbox}

Note that one studies the expected (instead of the actual) cost as the training process is assumed to be stochastic. In this section, we initiate a study of {\em when} this assumption holds. This is equivalent to understanding if an adversary can break the \pol protocol by creating a valid spoof with a lesser cost than honest training. \jia term such strategies as {\em stochastic spoofing}. We take a first step towards proving when stochastic spoofing {\em cannot} exist, \ie when the cheapness assumption cannot be violated. We do so by surfacing stability properties exhibited by stochastic spoofing adversaries but {\em not} by honest trainers. 

We begin by defining a class of algorithms that capture the goal of stochastic spoofing: algorithms that produce valid proofs with lower expected cost. This is formalized in Definition~\ref{def:cheap_pol}. Note that the non-existence of such algorithms is equivalent to the cheapness assumption, as formalized in Assumption~\ref{assum:cheapness}.

\begin{definition}[$c$-Cheap Proof Algorithm]
Consider $A_{D,W_T}$ as stated earlier in Section~\ref{subsec:primer}. Let $C$ be a cost function as defined in Definition~\ref{def:cost}, and $E = \mathbb{E}_{\mathcal{P} \sim A_{D,W_T}}[C(\mathcal{P})]$ be the expected cost over $A_{D,W_T}$ for some (honest) distribution given by (honest) probability measure $\mu$ on event space $A_{D,W_T}$. Then an algorithm $F: (D,W_T) \rightarrow A_{D,W_T}$ is $c$-cheap if $\mathbb{E}[C(F(D,W_T))]< cE, c \in [0,1)$.
\label{def:cheap_pol}
\end{definition}

This definition formalizes stochastic spoofing attacks $F$ as those whose expected cost is some fraction of the expected cost of honest training. The expectations are taken using some arbitrary probability measures for generality (\ie some measure representing honest training). 

{Having formally defined} spoofing attacks $F$, the main challenge in directly proving the (non-)existence of such spoofs is to identify whether access to a particular local minimum (\ie a trained model's weights) in a non-convex loss surface would enable stochastic spoofing. Without an answer to this question, we cannot say if spoofing attacks exist once the weights have been stolen.
Note that for convex optimization, where a {\em unique} global minimum exists, the adversary is guaranteed to reach the same minimum as the victim if the victim model is trained to convergence. However, this is usually not true under the setting of DNN training when the adversary only has access to the final state of a model (trained by an honest prover) but has no information about the rest of the proof: the loss landscape is often highly non-convex and has many local minima that can be attained through numerous valid paths~\cite{Choromaska2015TheLS}. 

However, instead of directly proving (non-)existence, we explore which properties $c$-cheap algorithms (for finding stochastic spoofs) must satisfy. This is a first step towards understanding when stochastic spoofing is possible and how one might introduce measures to prevent them.

\label{app:discussion_of_c_cheap}
\begin{lemma}[Stability of $c$-cheap algorithms]
\label{lem:stability}
Assuming $\mu$ represents the distribution from honestly training, a ``$b$-measure'' subset means honest training produces a proof in that set with likelihood $b$. A $c$-cheap algorithm, with probability $\frac{2}{3}$, only produces proofs in a ``$\leq \zeta$-measure'' subset of $A_{D,W_T}$ (using measure $\mu$), where $\zeta = \frac{Var(C(\mathcal{P}))}{(E(1-c) + a)^2}$ and $a = \sqrt{3Var(C(F(D,W_T))}$.
\end{lemma}

\begin{proof}
First by Markov's inequality we have $\mathbb{P}(|\mathbb{E}(C(F)) - C(F)|^2 > a^2) \leq \frac{Var(C(F))}{a^2}$, and so taking $a = \sqrt{3 Var(C(F))}$ we arrive at this upper-bound that is less than $1/3$, thus with probability $2/3$, $|\mathbb{E}(C(F)) - C(F)| < a$. 

Now the question is how many proofs are in that $a$ ball around $F$'s mean $E_1 = \mathbb{E}(C(F))$. This is given by $\mathbb{P}(C(\mathcal{P}) > E_1 - a) \mathbb{P}(C(\mathcal{P}) < E_1 + a) \leq \mathbb{P}(C(\mathcal{P}) < E_1 + a) \leq \mathbb{P}(|C(\mathcal{P}) - E|^2 \geq (E - E_1 + a)^2) \leq \frac{Var(C(\mathcal{P}))}{(E_0 - E_1 + a)^2} \leq \frac{Var(C(\mathcal{P}))}{(E_0(1-c) + a)^2}$ where the second last inequality was just Markov's inequality.
\end{proof}

Lemma~\ref{lem:stability} relates the likelihood of a stochastic spoofing attack $F$ producing a set of proofs to the likelihood honest training would produce those proofs. Particularly, when $\frac{Var(C(x))}{(E(1-c) + a)^2} < \frac{2}{3}$, Lemma~\ref{lem:stability} states that the likelihood of a set of proofs the adversary's algorithm $F$ generates is less than the likelihood that honest training produces those proofs. That is, a certain set becomes more common, or {\em more stable} with the adversary's algorithm. One could potentially use this information as an additional verification step to detect and reject such ({spoofing}) algorithms.

Complementing the enhanced stability properties of $F$, we also have a lower bound on the query complexity needed to obtain $c$-cheap proofs with honest training. Let us define $F_{honest}$ as the algorithm sampling/querying from $A_{D,W_T}$ with honest measure $\mu$ until it obtains a $c$-cheap proof. The following lemma gives a lower bound on how many queries $F_{honest}$ needs to achieve this. This is a potentially important property to set a baseline cost any prover must require to prevent stochastic spoofing.

\begin{lemma}[Queries]
\label{lem:queries}
Let $F_{honest}$ query $A_{D,W_T}$ (with probability measure $\mu$) inducing a distribution on $C(\mathcal{P}), \mathcal{P}\in A_{D,W_T}$. Then with probability $\frac{2}{3}$, $F_{honest}$ issues $\geq \frac{\log(1/3)}{\log(1-P)}$ queries where $P = \frac{Var_{\mu}(C(\mathcal{P}))}{(1-c)^2E^2}$ to obtain a $c$-cheap proof.
\end{lemma}

\begin{proof}
We drop the $\mu$ subscripts
{but note this is}
the measure for cost distribution. We have $\mathbb{P}(C(\mathcal{P}) \leq cE) \leq \mathbb{P}(|E-C(\mathcal{P})|^2 \geq (1-c)^2E^2) \leq \frac{Var(C(\mathcal{P}))}{(1-c)^2E^2} \coloneqq P$, where the last inequality was by Markov's inequality. Thus $\mathbb{P}(C(\mathcal{P}) > cE) \geq 1 -P$.

{We} are interested in $N$ s.t{.} $(1-P)^N \leq 1/3$ so that with probability $2/3$ we obtain a $c$-cheap proof after $N$ queries if $\mathbb{P}(C(\mathcal{P}) > cE) = 1 -P$; this establishes a lower-bound as in general $\mathbb{P}(C(\mathcal{P}) > cE) \geq 1 -P$. This is simply given by $N = \frac{\log(1/3)}{\log(1-P)}$ concluding the proof.
\end{proof}

The previous two lemmas have identified conditions that $c$-cheap algorithms would need to satisfy. Whether an algorithm can satisfy these properties is an open problem, and a negative answer would also prove that $c$-cheap algorithms do not exist. However, stability and query complexity are only two properties that stochastic spoofing adversaries must satisfy, \ie they are necessary but may not be sufficient. 
It is worth noting that obtaining the identical model $W_T$ via training is equivalent to knowledge transfer without any loss of information, whereas most existing knowledge transfer algorithms only preserve the models' behavior on task data distribution~\cite{hinton2015distilling}. Thus $c$-cheap algorithms may be considered as a type of special knowledge-transfer algorithm.
In general, it remains 
an open problem whether algorithms given by Definition~\ref{def:cheap_pol} can or cannot exist (under what update rules, verification schemes, etc.) for DNN training. This will dictate whether (or when) comprehensive defense strategies against spoofing adversaries exist.  

\begin{tcolorbox}
\bf Open Question(s): Definition~\ref{def:cheap_pol} formally defines stochastic spoofing as algorithms that can create valid proofs with a lesser cost than honest training (given knowledge of the minimum from another training run). Do such algorithms exist, and under what conditions will they exist (or not)? 
\end{tcolorbox}

As a first step to better understand this open problem, we will introduce and evaluate several examples of candidate (stochastic spoofing) algorithms in Section~\ref{ssec:stochastic_spoofing_experiment}. This will enable us to evaluate their performance/computational costs empirically.

\section{{Empirical Evaluation of Efficient Verification of Valid Proofs}}
\label{sec:fundamentals_revisited}

In this section, our goal is to empirically explore the theoretical claims made on the robustness and efficiency of proof verification (see  Section~\ref{sec:giving_guarantees_with_pol}), specifically, the reproducibility and representativeness assumptions. Building on the findings in Section~\ref{ssec:good_training}, where we identified that imprecise tolerance to noise is a vulnerability in proof verification, we introduce a novel attack that outperforms prior attacks targeting this vulnerability (see Section~\ref{subsec:spoofs}). Additionally, we propose an attack against the top-$Q$ selection mechanism showing that it indeed does not identify representative subsets. Both attacks belong to the category of structurally correct spoofs, which means the adversary creates an invalid proof but the proof passes the verification by either targeting the (non-optimal) choice of noise tolerance or the representative subset.

\vspace{2mm}
\noindent{\bf Experimental Setup:} For the following experiments, we use the same setup from \jia: we evaluate all following experiments using CIFAR-10 and CIFAR-100~\cite{Krizhevsky09learningmultiple} datasets. To ensure a fair comparison with prior work~\cite{jia2021proof}, we used ResNet-20 and ResNet-50~\cite{he2015resnet} as the model architecture for the two tasks{,} respectively, and trained the models with a batch size of 128 for 200 epochs. Unless specified, all experiments are repeated 5 times and the figures include the confidence interval.

\subsection{On the Reproducibility in \pol} 
\label{ssec:on_finding_cheap_pols}

Recall that in Section~\ref{ssec:good_training}, we analyzed one of the fundamental assumptions for the \pol protocol: reproducibility of gradient updates. We proved the existence of optimal per-step verification thresholds, but due to the noise encountered while performing the computations, it remains an open problem as to how to construct the optimal verification strategy. Prior work by \zhang , discussed in Section~\ref{subsec:spoofs}, has implicitly exploited the vulnerability of static thresholds in verification to create spoofs (\ie they did not point out this vulnerability).

However, their attack is computationally costly, requiring at least $(43 \cdot n_{\text{iter}} + 1) \cdot k$ forward passes (FPs) for every update in the proof, where $n_{\text{iter}}$ is the number of iterations for optimizing the adversarial example (more details are in Appendix~\ref{app:cost}). Furthermore, it is not guaranteed to converge, so \zhang had to assume that the adversary is able to manipulate \pol hyperparameters such as the checkpointing interval (see Section~\ref{ssec:limitations_prior_spoofing}).\footnote{We reproduced their results and found it does not work
{when the checkpointing interval ($k$) differs from the ones picked in \cite{zhang2021adversarial}} (\eg $k=10$ for CIFAR-100).} In this subsection, we will introduce a new attack that targets the same vulnerability: the {\em infinitesimal update attack}. Our new attack is more efficient and is guaranteed to succeed without assuming control of the \pol hyperparameters. 

\begin{figure}[t]
\centering\hspace{-5mm}
\subfloat[Valid Update]{\includegraphics[height=3cm]{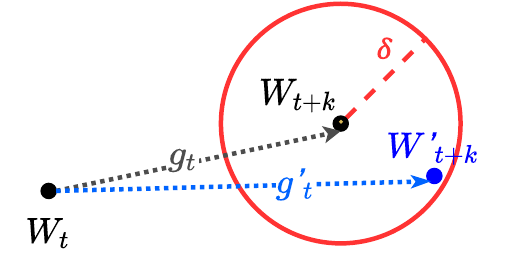}}
\vspace{5mm}

\subfloat[Spoofed Update]{\hspace{26mm}\includegraphics[height=3cm]{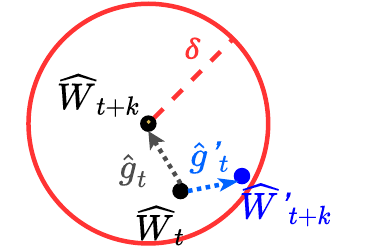}}
\caption{ {\bf Illustration of valid and spoofed updates in the infinitesimal update attack: } Usually, the proof update, $g_t$, and the verifier-reproduced update, $g'_t$, would both be much greater than $\delta$. The discrepancy between $g_t$ and $g'_t$ should solely come from hardware/software induced noise. However, in the infinitesimal update attack, the update from the spoof, $\hat{g}_t$, and the verifier-reproduced update, $\hat{g}'_t$, are both much smaller than $\delta$, which causes their difference to also be smaller than $\delta$. Hence, the spoof passes the verification protocol proposed by \jia. This  illustrates how an adversary can exploit the fixed $\delta$ to create an invalid \pol that passes verification.}
\label{fig:attack}
\end{figure}

\vspace{2mm}
\noindent\textbf{Infinitesimal update attack.} Let us re-establish{our notation. Let $g_t = W_{t+k} - W_{t}$ be the update in the honest proof, and $\hat{g}_t = \hat{W}_{t+k} - \hat{W}_{t}$ the update in the adversary's spoof. Let $g'_t = W'_{t+k} - W_{t}$ and $\hat{g}'_t = \hat{W}'_{t+k} - \hat{W}_{t}$ denote the corresponding reproduced updates by the verifier.

At a high level, the static threshold vulnerability can be exploited by small-magnitude update, as shown in Figure~\ref{fig:attack}. To demonstrate this, we propose a strategy that (a) requires near-zero computational cost; (b) is guaranteed to yield updates of near-zero magnitude; and (c) is hard for the verifier to detect. The idea is simple: to obtain an update of near-zero magnitude, one can either generate a near-zero gradient, or one can use an {\em infinitesimal learning rate}. Formally, the strategy is as follows:

\begin{enumerate}
\item Generate model weights $\hat{W}_k$, $\hat{W}_{2k}$ between $\hat{W}_0$ and the victim model $W_T$ such that, $d(\hat{W}_{i \cdot k}, \hat{W}_{(i+1)\cdot k}) \ll \delta$. As an exemplar approach, we utilize linear interpolation to achieve this. 
\item The learning rate $\eta$ is set to a small value (\ie $\eta \rightarrow 0$) such that the update is always smaller than $\frac{\delta}{k}$ irrespective of the gradient value.
\item All other information logged, including the data, can be random values. 
\end{enumerate}

It is clear that the cost of infinitesimal update attack is low since no training is required. The only cost comes from linear interpolation in step 1, which requires one floating point operation for every model parameter per model update. This amount is upper-bounded by the computation needed for 1 FP, and is much less than the aforementioned cost of Attack 2 by \zhang. Moreover, this attack is hard for the verifier (with the static threshold) to detect as it is always possible to create a much smaller update compared to $\delta$. If the verifier decreases $\delta$ to detect such updates,  valid updates may also be discarded, resulting in a high false negative rate. Note that detecting linear interpolation is not sufficient either because it is not the only way to generate models:  any strategy such that $d(\hat{W}_{i \cdot k}, \hat{W}_{(i+1)\cdot k}) \ll \delta$ $\forall i$ will be similarly effective. 

\begin{figure}[t]
\centering
\subfloat[{Infinitesimal} update attack (CIFAR-10)
\label{subfig:infinitesimal}]
{
\includegraphics[width=0.8\linewidth]{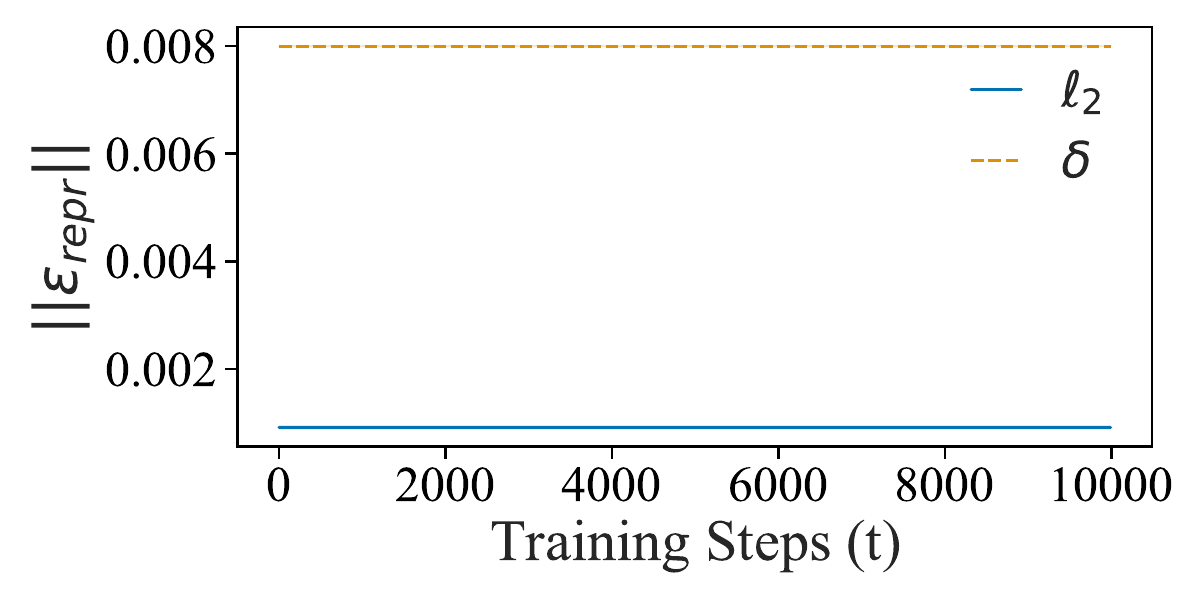}
}
\\
\subfloat[Attack 2 by Zhang ~\etal~\cite{zhang2021adversarial} (CIFAR-10)
\label{subfig:zhang}]
{
\includegraphics[width=0.8\linewidth]{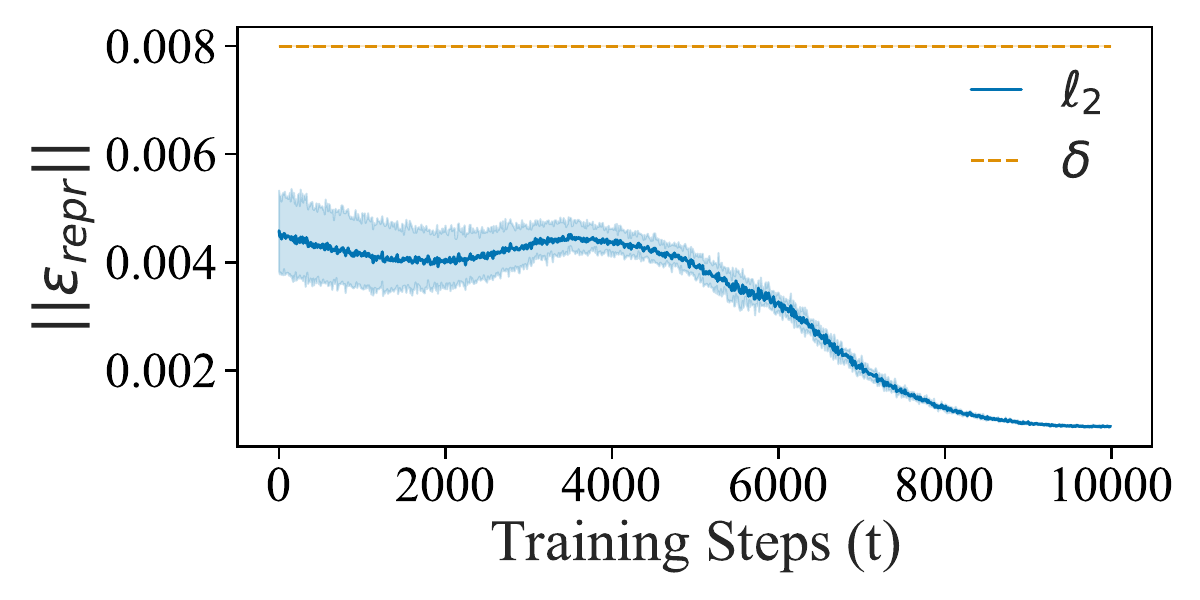}
}
\\ %
\subfloat[Infinitesimal update attack (CIFAR-100) \label{subfig:infinitesimal_cifar100}]
{\includegraphics[width=0.8\linewidth]{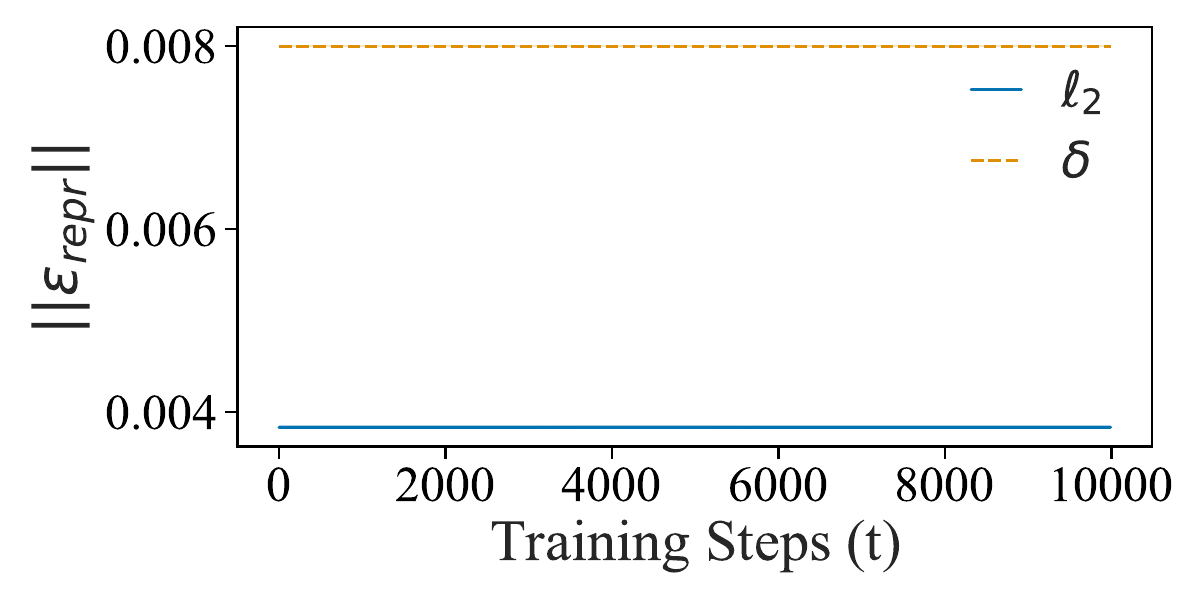}
}
\\
\subfloat[Attack 2 by \zhang (CIFAR-100)\label{subfig:zhang_cifar100}]{
\includegraphics[width=0.8\linewidth]{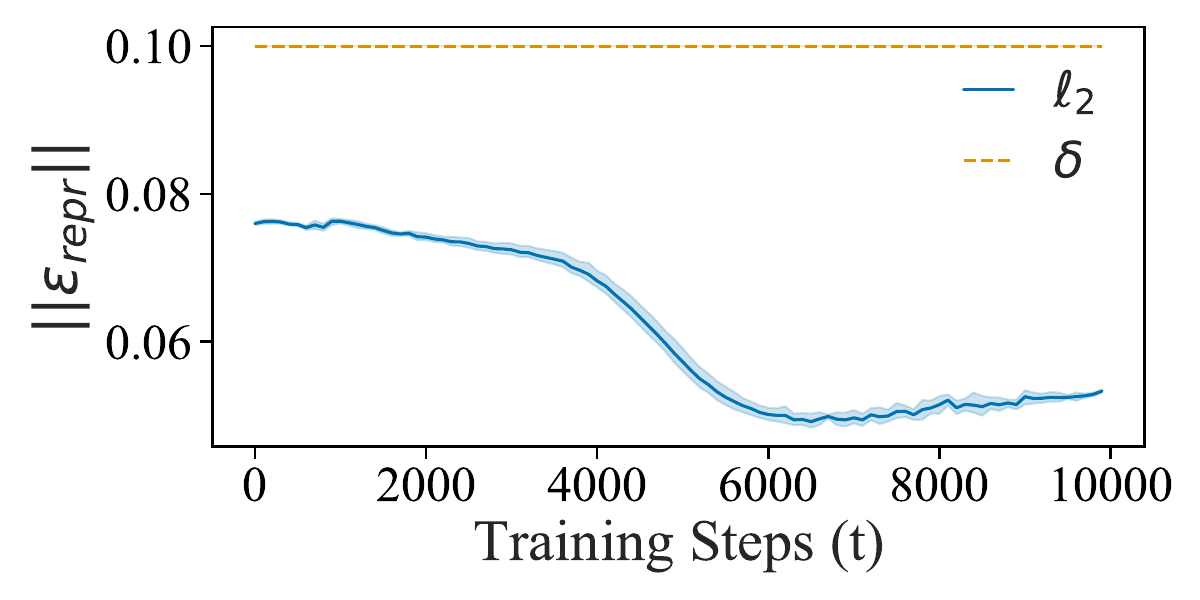}
}
\caption{{\textbf{The infinitesimal update attack and Attack 2 by \zhang can bypass verification by exploiting the imprecise $\delta$ bound. }$\delta$ is set to 0.008 manually in these experiments. It is observed that \nre is much smaller than $\delta$ for all steps, indicating that a static $\delta$ is not ideal for verification. We further observe that in addition to utilizing lesser computational resources, infinitesimal update generates consistently lower \nre than Attack 2.}}
\label{fig:baseline}
\end{figure}

\vspace{2mm}
\noindent{\em Evaluation \& Results:} We evaluated  our infinitesimal updates strategy against the original \pol framework with the same experimental setup (described at the beginning of this section) and the same parameters as \jia. We assume the most powerful verifier: verifying {\em all updates}. 

It can be seen in Figure~\ref{subfig:infinitesimal} and~\ref{subfig:infinitesimal_cifar100} that our proposed strategy is able to achieve a normalized reproduction error \nre significantly smaller than $\delta$, and can thus pass verification. We re-implement and reproduce Attack 2 by \zhang (refer to Figure~\ref{subfig:zhang} and~\ref{subfig:zhang_cifar100}) as a baseline to understand the effectiveness of our approach. It is observed that our attack always outperforms that of \zhang. This observation is true even when we later utilized code from the repository of \zhang.

\subsection{On the Representativeness in \pol} 
\label{ssec:efficient_verification}

As discussed earlier, an efficient verification mechanism that fails to select a representative subset of training updates may jeopardize the role of validity detection played by proof verification. This is especially true when the selection mechanism does not meet certain properties, \ie failing to include at least one invalid update in the representative subset when invalid updates exist in the proof, as discussed in Section~\ref{sssec:topq_theory}. Here we introduce an attack against the top-$Q$ verification mechanism to illustrate how an adversary can exploit this. By doing so, we demonstrate the necessity of the aforementioned properties for correct and efficient verification.

The selection mechanism for top-$Q$ updates makes verification more efficient by 
reducing the number of updates that need to be verified. This opens an attack surface for an adversary who now only needs to ensure that (a) it controls the selected updates; and (b) they pass verification. At a high level, one possible strategy is for the adversary to manipulate the magnitude of updates (\eg by using a large learning rate $\eta$) to control which updates are verified. More formally, the {\em blindfold top-$Q$} strategy we instantiate is as follows:
\begin{enumerate}
\item Generate the list of model checkpoints $\hat{W}_S$, $\hat{W}_{2S}$, $\cdots$ $\hat{W}_{T-S}$ by interpolating linearly and evenly between $\hat{W}_0$ and $W_T$, where $S$ is the number of updates per epoch.
\item For each epoch $i$, start with $\hat{W}_{i \cdot S}$, create $k \cdot Q$ valid updates by applying SGD with a large learning rate $\eta$. Store $\hat{W}_{i \cdot S+k}$, $\hat{W}_{i \cdot S+2k}$, $\cdots$ $\hat{W}_{i \cdot S+Q \cdot k}$. Record the magnitudes of these updates.
\item The remaining model states, \ie $\hat{W}_{i \cdot S+(Q+1) \cdot k}$ $\cdots$ $\hat{W}_{(i+1) \cdot S-k}$ can be created by linearly interpolating between $\hat{W}_{i \cdot S+Q \cdot k}$ and $\hat{W}_{(i+1) \cdot S}$, as long as these updates are smaller than any of the $Q$ valid updates computed in the previous step. This is similar to what is done in the infinitesimal update attack.
\end{enumerate}

To analyze the per-step cost of this attack, note there are two cases: the update can be either a top-$Q$ update or not. In the former case, besides linear interpolation, $k$ valid gradient updates
{need to be computed, where each costs 1 FP and 1 backward pass (BP), or approximately a total of 3 FPs for each update.} For the latter, only linear interpolation is needed. Therefore, the expected step-wise cost is $(Q/s) \cdot (3 \cdot k + 1) + (s-Q)/s = (3 \cdot k \cdot Q) / s + 1$ FPs, where $s$ is the number of steps per epoch.

Observe that the blindfold top-$Q$ strategy exists only because the verifier deployed a heuristic to decrease the computational cost of verification; such heuristics provide weaker (probabilistic) guarantees in comparison to verifying the entire proof when the updates selected for verification do not form a representative subset. The selection mechanism for top-$Q$ updates was proposed to provide a better trade-off than a strategy that involved randomly selecting updates. Alternative approaches that do not factor the representativeness assumption can suffer similar pitfalls; this allows the adversary to hide invalid updates among the ones that will be verified.

\vspace{2mm}
\noindent{\em Evaluation \& Results:} Following the same experimental setup as \jia, and taking $Q$ to be 5, we assumed the adversary obtained the chosen $Q$ and thus performed the blindfold top-$Q$ strategy by having 5 valid updates per epoch. As shown in Figure~\ref{fig:new_top_q_metric}, the adversary is always able to mislead the verifier into selecting these updates.

\section{{Empirical Evaluation of the Cheapness Assumption}}
\label{ssec:stochastic_spoofing_experiment}

\begin{figure}[t]
\centering
\subfloat[CIFAR-10]{\includegraphics[width=0.8\linewidth]{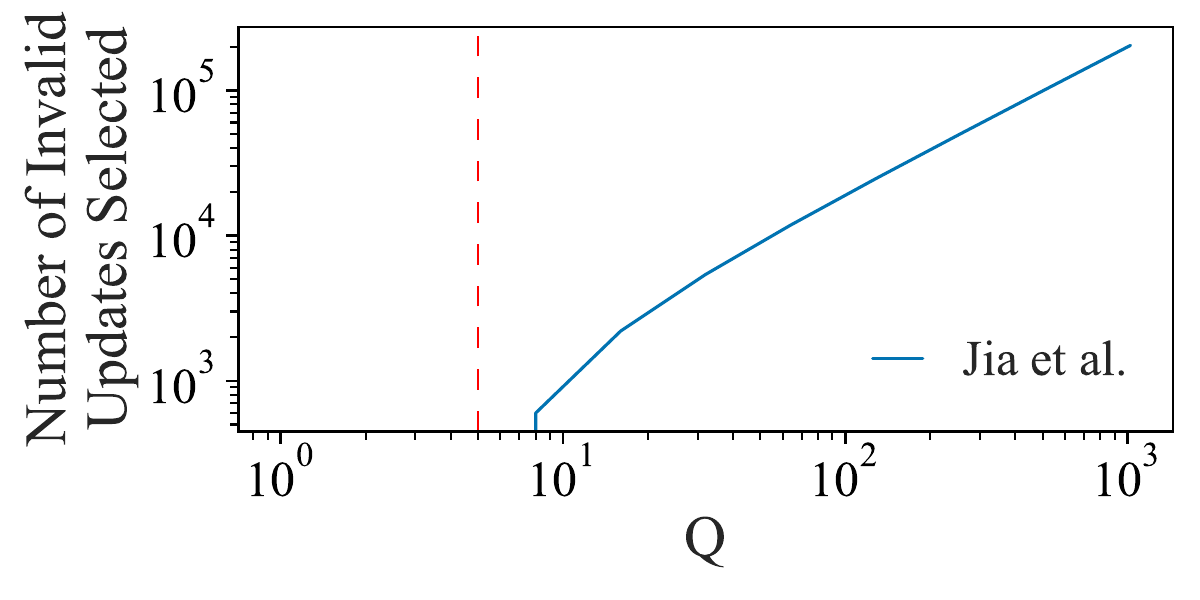}}

\subfloat[CIFAR-100]{\includegraphics[width=0.8\linewidth]{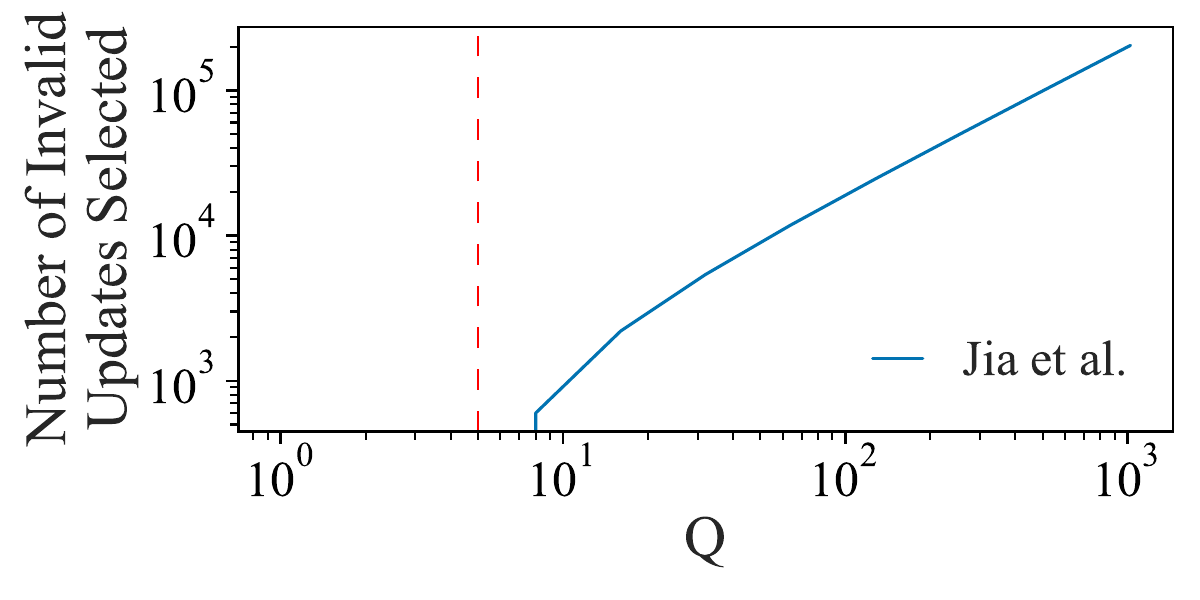}}
\caption{ \textbf{Evaluation of blindfold top-$Q$ attack} (see Section~\ref{ssec:efficient_verification}). We implement this attack for $Q=5$ as shown by the dashed red line. It is assumed that the adversary knows $Q$ thus they submit only $5$ large but valid updates that are constructed to pass verification. The rest of the proof is computationally cheap, and constructed to be invalid. Observe that the top-$Q$ method of \jia fails to detect the invalid updates (the blue curve is always 0 left of the red line) as long as the verifier does not use a larger $Q$ than what is previously claimed.}
\label{fig:new_top_q_metric}
\end{figure}

{The previous sections discussed how structurally correct spoofs violate} the first security assumption of the verification protocol by targeting the noise tolerance and the heuristics for efficient verification. In this section, we will explore attacks that target the second role of the verification protocol, which is to establish precedence. It should be computationally cheaper to obtain a proof from training than a spoof post hoc even when given access to the trained model's weights. Adversaries targeting this aspect of the \pol protocol are termed stochastic spoofing adversaries. In this section, we empirically analyze the difficulty of successfully constructing such a stochastic spoof. Although this analysis does not solve the problem of formally proving the (non-)existence of stochastic spoofing (see Section~\ref{sec:theory_cheap_spoofing}), it highlights that a better understanding of the convergence of optimizers on non-convex loss surfaces is required to instantiate stochastic spoofing adversaries. In other words, if stochastic spoofing attacks are to be more successful it would likely result from developments in learning theory that lead to the invention of new optimizers.

Note that in our experiments, we focus on DNNs because: (a) DNNs are more likely to be targeted by adversaries because they are computationally expensive to train compared to convex models; and (b) it is unknown whether the same local minimum in a DNN's non-convex loss surface can be recovered given knowledge of that particular minimum. Instead, for convex models, there exists a unique global minimum that an honest trainer achieves, and the adversary is guaranteed to achieve this exact same minimum.

\subsection{Why a Stochastic Spoofing Adversary Needs to Know the Final Model Weights} 

To understand the difficulty of recovering the same local minimum for DNNs through honest training (\ie \xspace {\em without} utilizing the knowledge of the final weights of another model), we empirically check if two nearly identical training setups can lead to the same weights. We kept the architecture, optimization algorithm, and training data the same but varied only the randomness in initialization and data sampling. We trained multiple models independently until convergence and compute the pairwise $\ell_2$ distance (of 30 data points we collected) of their weights, as shown in Table~\ref{tab:avg_dist}. We found with high consistency (low standard deviation) that this distance was large \ie on the order of the $\ell_2$ norm of the weights themselves ($55.481 \pm 0.110$ for ResNet-20 and $44.624 \pm 0.196$ for ResNet-50), and {\em not due to hardware noise} as it substantially exceeded \re. To determine if this is significant, we then performed a one-tailed $t$-test with null hypothesis that the distance between independent model parameters was zero. The $p$-values, summarized in Table~\ref{tab:avg_dist}, indicate that we could reject the null hypothesis with high confidence (low $p$-values). Given that such a minimal change could consistently result in significantly different weights, we argue that it is highly improbable to recover the same final weights under realistic scenarios (with even larger setup differences). Therefore, knowledge of the final model weights is essential to a stochastic spoofing adversary.

\begin{table}[t]
\centering
\normalsize
\begin{tabular}{l | c c }
\toprule
{\bf Setup} & $\ell_2$ {\bf distance} & {\bf $p$-value} \\ 
\midrule
\midrule
ResNet-20 & \multirow{ 2}{*}{$71.472 \pm 0.186$}   & \multirow{ 2}{*}{$8.65\times10^{-77}$}\\
CIFAR-10 & & \\
\hline
ResNet-50 & \multirow{ 2}{*}{$58.056 \pm 0.205$} & \multirow{ 2}{*}{$5.43\times10^{-73}$}\\
CIFAR-100 &  & \\
\bottomrule
\end{tabular}
\caption{ \textbf{Average distance of converged models for independent training runs on the same architecture.} A one-tailed $t$-test is used to test the null hypothesis: the models have distances of 0, and we report the $p$-values.}
\label{tab:avg_dist}
\end{table} 

\begin{figure}[t]
\centering
\subfloat[Epoch 0 (CIFAR-10)]{\includegraphics[width=0.83\linewidth]{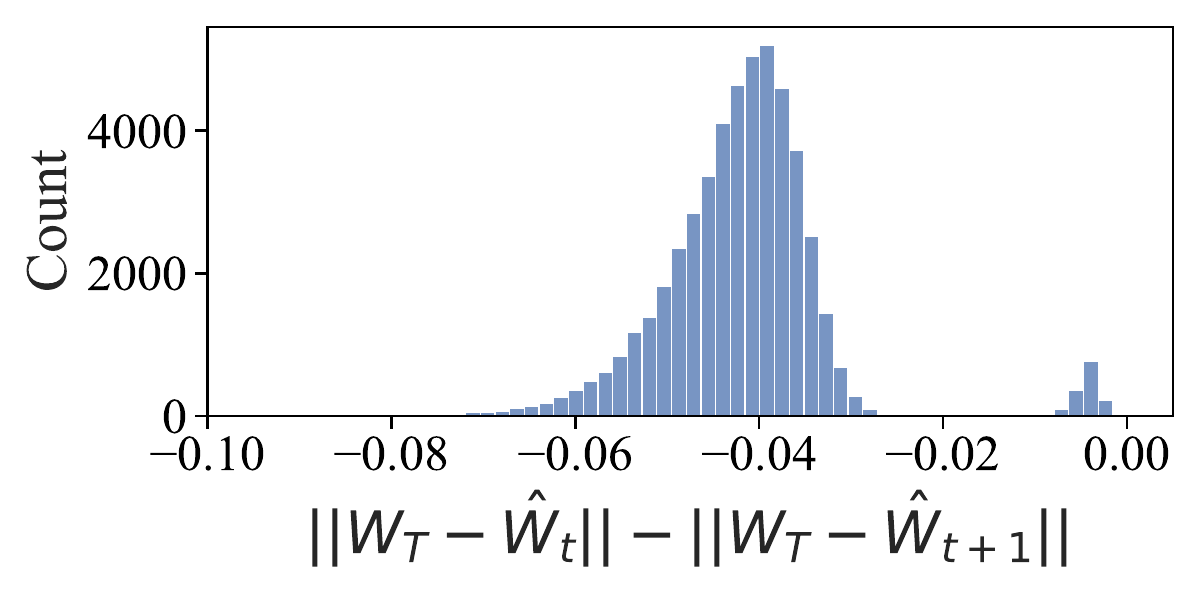}}

\subfloat[Epoch 200 (CIFAR-10)]{\includegraphics[width=0.83\linewidth]{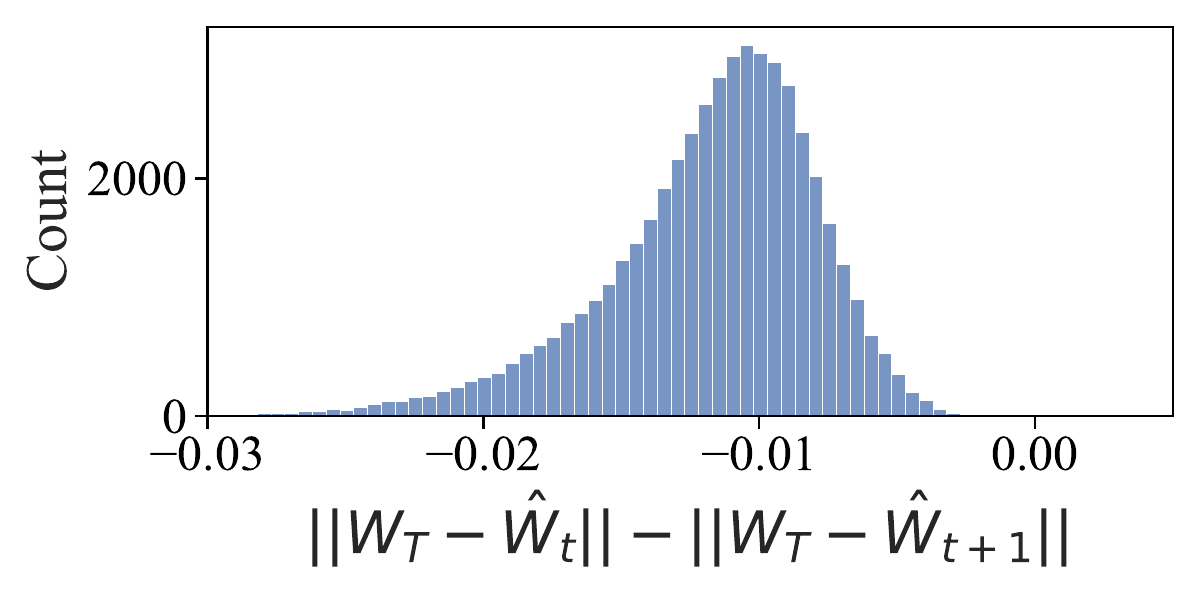}}

\subfloat[Epoch 0 (CIFAR-100)]{\includegraphics[width=0.83\linewidth]{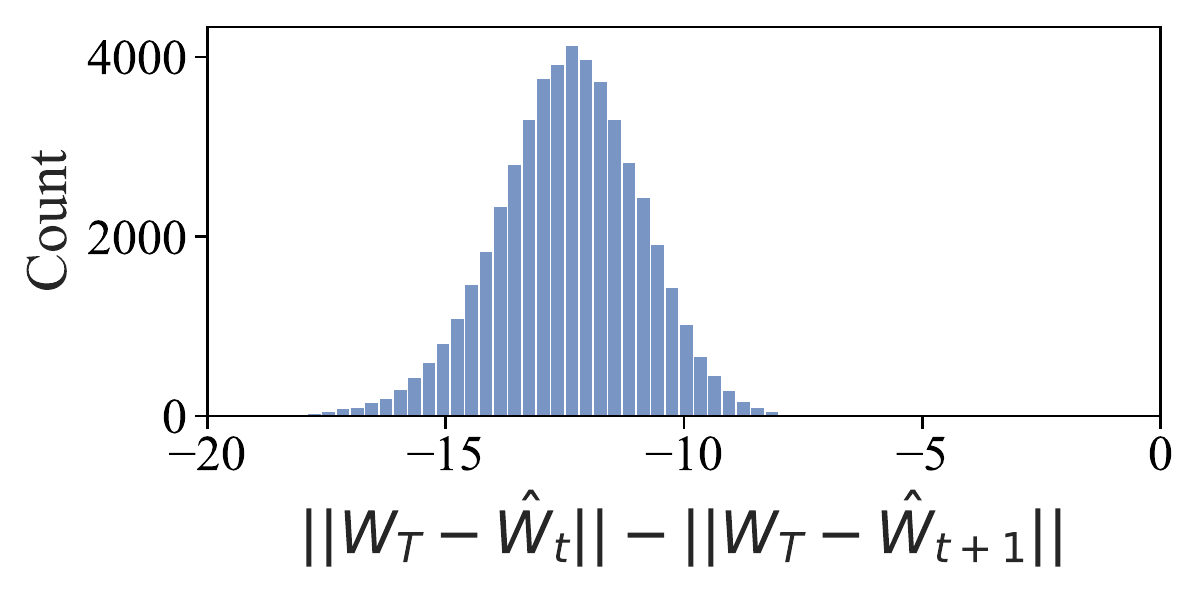}}

\subfloat[Epoch 200 (CIFAR-100)]{\includegraphics[width=0.83\linewidth]{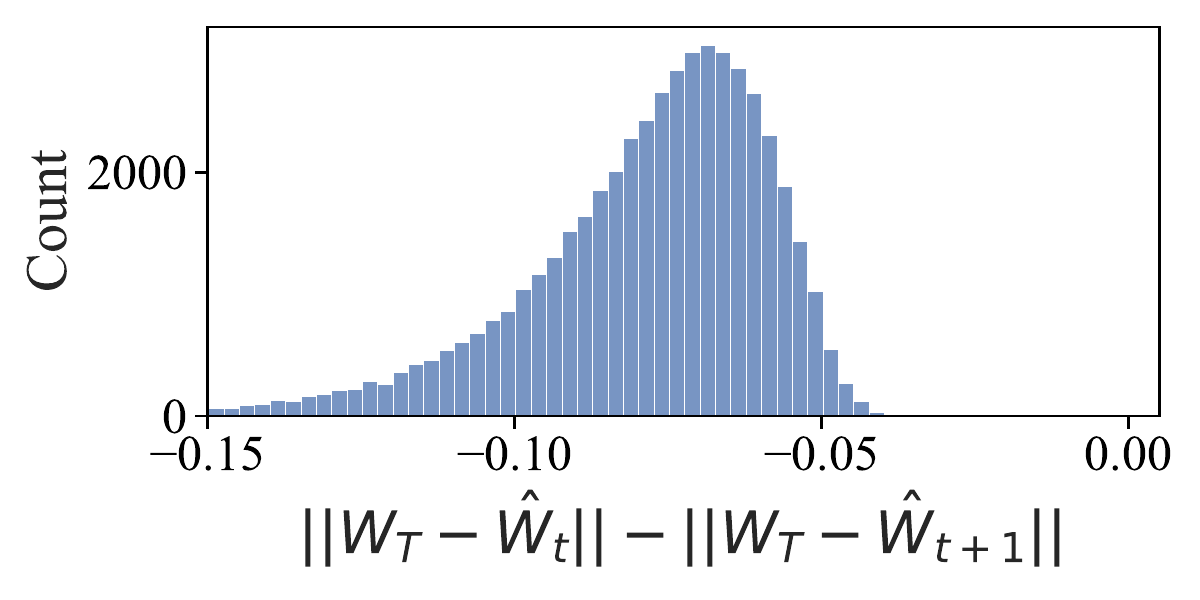}}
\caption{ \textbf{No update brings the adversary's model closer to the honest prover's with different initializations.} At any step, an adversary can select its next update ($\hat{W}_{t+1}-\hat{W}_t$) as a linear combination of the updates from any individual data points. We visualize if any update brings an adversary's intermediate state $\hat{W}$ closer to the honest prover's final model $W_T$. Indeed, with all updates negative in the above four histograms, it is impossible (at epoch 0 and 200 in the training process) for an adversary to force its model closer to the prover's using data from the prover's training distribution. Detailed results for adversary's models pre-trained for different number of steps are included in Figure~\ref{fig:boa_cifar10_leftovers} and Figure~\ref{fig:boa_cifar100} in Appendix~\ref{app:additional_figures}.}
\label{fig:hist}
\end{figure}

\subsection{Adversarial Reconstruction of a Proof for Known Model Weights}
\label{sssec:adv_reconstruct}

{We also examine spoofing strategies that attempt to construct} valid proofs ending in $\hat{W}_{\hat{T}}$ ({such that} $\hat{W}_{\hat{T}} = W_T$) using prior knowledge of $W_T$ ({weights of the victim model}). In other words, such adversaries attempt to direct legitimate gradient updates toward the desired victim model by manipulating the training data. We analyze two classes of spoofing strategies wherein the adversary uses: (a) original training data which is reordered by using data ordering attacks~\cite{DBLP:journals/corr/abs-2104-09667}; and (b) (generated) synthetic data. In studying these strategies, we assume a powerful adversarial model with access to the training dataset, as defined in Section~\ref{subsec:threat_model}. However, we emphasize that the protocol may be more robust by preventing the adversary from accessing the training distribution (\eg by keeping the dataset private) as mentioned in Section~\ref{subsec:threat_model}; we will discuss this in more detail in Section~\ref{sssec:data_commitment}). For now, we keep assuming this access and provide a stricter/more realistic assessment of security from the defender's perspective. 

\vspace{2mm}
\noindent{\bf 1. Data Ordering Attacks:} The adversary changes the order of the mini-batches during the training process to obtain a desired gradient update~\cite{DBLP:journals/corr/abs-2104-09667}. Through such attacks, we wish to understand the difficulty for an adversary to generate gradient updates using a new initialization ($\hat{W}_0$), resulting in a distinct and unrelated training trajectory to the victim model $W_T$. In this scenario, the adversary's objective is to minimize the distance between the parameters of its model and those of the victim (\ie $\min ||W_T - \hat{W}_{\hat{T}}||$) by reordering the training data. 

We know that a gradient update for a mini-batch is the mean of the gradient updates of the individual data points in the mini-batch. Therefore, we first evaluate how likely it is for the gradient update of a single data point to move the adversary's model closer to the victim's. We perform the following experiment: we update the weight, $\hat{W}_t$ by SGD with one data point from the training dataset to get to $\hat{W}_{t+1}$, and we compute the change in $\ell_2$ distance to the victim weights $W_T$ (\ie if $|| W_T - \hat{W}_t|| - || W_T - \hat{W}_{t+1}||>0$). This is repeated for every training data point of the dataset, and the results are plotted as a histogram in Figure~\ref{fig:hist}. We observe that irrespective of whether $\hat{W}_t$ was freshly initialized (\ie $\hat{W}_t = \hat{W}_0$) or if it was pre-trained for 200 epochs (using the victim's weights as the labeling oracle), none of the data points results in an update that leads to $W_T$. More detailed results for different amounts of pre-training are included in Figures~\ref{fig:boa_cifar10_leftovers} and~\ref{fig:boa_cifar100} in Appendix~\ref{app:additional_figures}.

{This analysis suggests that} data ordering attacks are unlikely to be successful for stochastic spoofing. It is worth noting that previous research has tackled the question of reconstructing updates transitioning from a known  weight initialization to another known weight from the same training run~\cite{DBLP:journals/corr/abs-2104-09667,thudi2021necessity}. However, in this case, our adversary is attempting to reconstruct the final weight from an {\em unknown} initialization.

\begin{figure}[t]
\centering
\subfloat[CIFAR-10]{\includegraphics[width=0.95\linewidth]{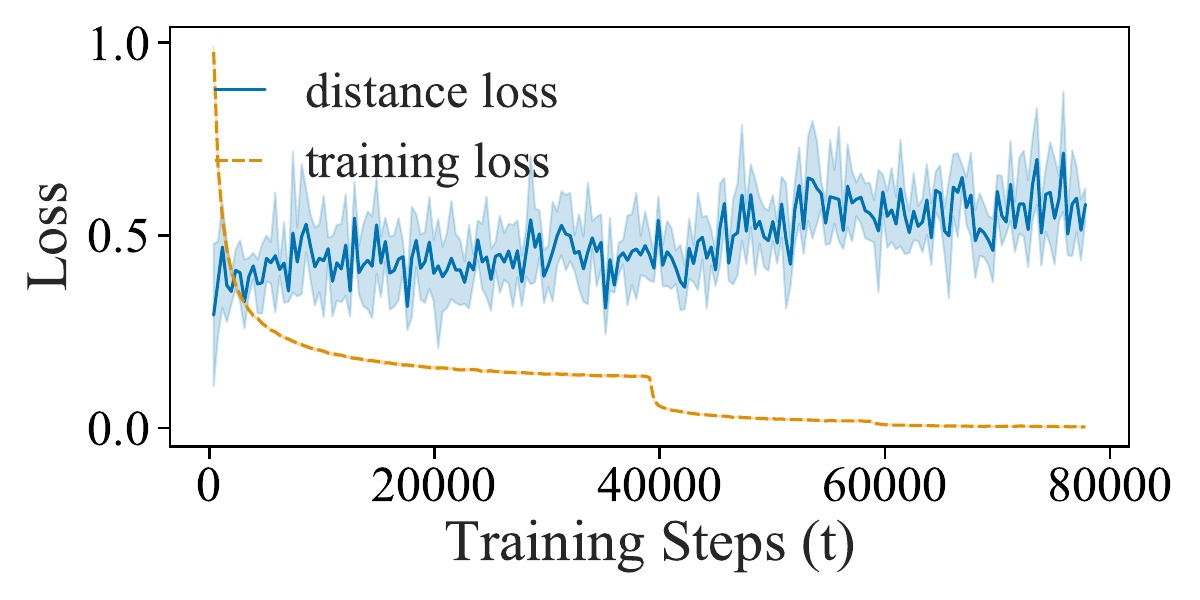}\label{fig:amplifier_intuition_legit}}

\subfloat[CIFAR-100]{\includegraphics[width=0.95\linewidth]{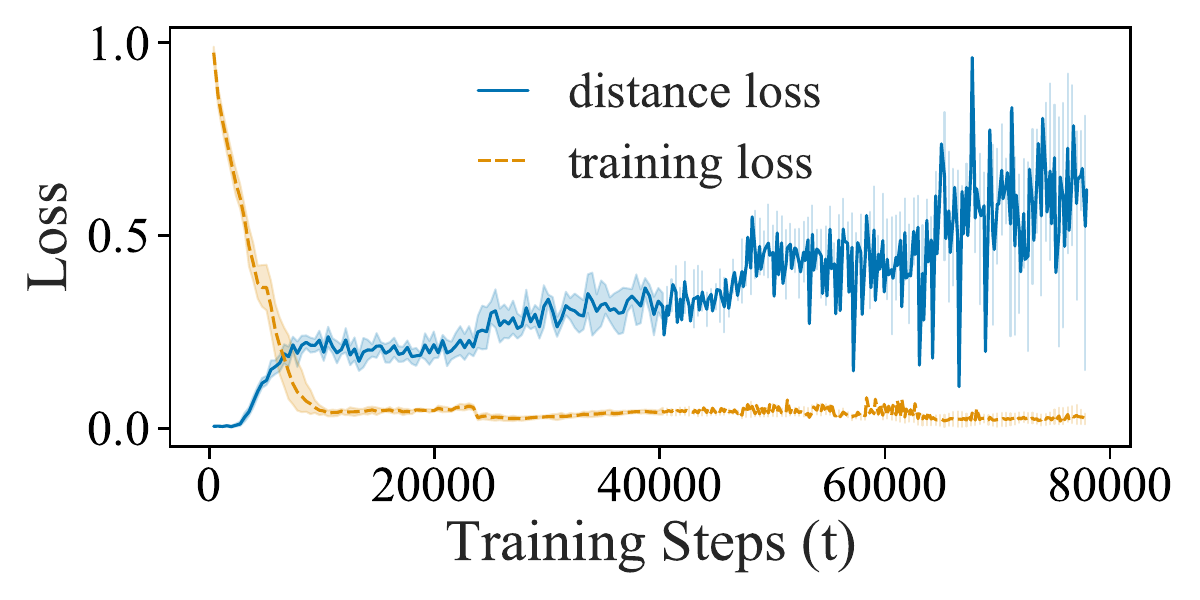}\label{fig:amplifier_intuition_spoof}}
\caption{ \textbf{Evolution of the model's loss when synthesizing adversarial data.} An adversary simultaneously optimizes a standard training penalty and a penalty designed to minimize the distance between the model and a target victim model $W_T$. Training is performed on synthetic data constructed to decrease distance to the victim model. We normalize both penalties to the range of $(0, 1)$ for readability. The trends of the two curves are opposite to each other. Thus, this attack is unlikely to obtain a model close enough to the victim model.
}
\label{fig:contradicting_loss}
\end{figure}

\vspace{2mm}
\noindent{\bf 2. Synthesizing Adversarial Data:} Instead of changing the mini-batch ordering, one can change the data points themselves. To this end, the problem can be formulated as follows: find a dataset that, upon training, results in final weights that are close to $W_T$. Below is one formulation of the problem:
\begin{equation*}
\hat{D}^* = \argmin_{\hat{D}} ||W_T - \argmin_{\hat{W}} \mathcal{L}(\hat{W}, \hat{D}) ||, 
\end{equation*} 
Note that $\mathcal{L}(\hat{W}, \hat{D}) $ is the loss evaluated on the model weights $\hat{W}$ and the dataset, $\hat{D}$. 
\footnote{This attack is different from Attack 1 of \zhang in the sense that we iteratively optimize each batch of data to decrease $||W_T - \hat{W}||$ while they focus on one batch of data until $||\hat{W_t} - \hat{W}|| \approx 0$ for some $\hat{W_t}$ created based on $W_T$.} 
The adversary cannot directly optimize the distance between their weights and the stolen weights (\ie $\min_{\hat{W}} ||W_T - \hat{W}||$). This is because this loss is not used for training, and including the loss term (which contains $W_T$) in the proof would essentially be an admission that the adversary is attempting to spoof the proof for a stolen model~\cite{jia2021proof}. We empirically evaluate this approach by synthesizing the adversarial dataset using gradient descent (\ie $\hat{D}^* = \argmin(||W_T - \hat{W}^*||)$) and update the adversary's model $\hat{W}^* = \argmin(\mathcal{L}(\hat{W}, \hat{D}))$ in an alternating manner. As shown in Figure~\ref{fig:contradicting_loss}, the experiment results show that the two losses often oppose each other (\ie one of the loss curves is increasing while the other is decreasing)~\cite{hansen1992new}. Though this does not disprove the possibility of further manipulating the losses to allow training, it is non-trivial and suggests possible incompatibility guarantees, as shown in Figure~\ref{fig:contradicting_loss}.

\vspace{2mm}
\noindent{\bf Summary:}  We demonstrate the importance of knowing the final model weights in bootstrapping a stochastic spoofing adversary. We evaluated two types of adversarial reconstructions of proofs based on stochastic spoofing. This allowed us to explore how difficult it is to converge to $W_T$ with valid gradient updates given knowledge of $W_T$. We observed that it is challenging to take a model state closer to another model state (obtained from a different random initialization) by simply reordering the training data points. Hence we designed experiments to evaluate whether it is possible to make a valid training trajectory that ends at the desired final weights via adversarial steps. We emphasize that for all of the experiments above we assumed the adversary has access to the training data. This capacity is assumed to obtain an empirical assessment closer to the worst-case adversary. Beyond data access, in Appendix~\ref{app:rna_attack}, we illustrate a spoofing strategy with the assumption that the adversary has the capacity to manipulate the internal states of the optimizer. Despite these additional capabilities, it is still empirically hard for the adversary to generate a valid sequence that reaches the exact weight $W_T$ while at the same time expending less computational power than a legitimate training run. This is because the adversary does not have the required knowledge from the honest prover's training run, such as the initial model state, the order of training data, and all other metadata which is contained in the PoL. Based on these empirical observations, we believe that such algorithms do not currently exist for DNN training. However, we note that any empirical study is non-comprehensive as other attacks may exist. Their development is likely to come with/from progresses in learning theory.

\section{Discussion}

\label{ssec:fs_attacks}

So far, we have formalized and empirically validated the necessary assumptions for proof verification to play its first role---efficient verification of valid proofs. We have also taken steps towards a formal understanding of the difficulties that surround
ensuring that the verification mechanism provides precedence: stochastic spoofs that are cheaper than honest training should not exist. An answer based on statistical learning to the latter remains elusive. In this section, we instead discuss how practical instantiations of PoL can leverage security primitives to circumvent these issues with non-ML solutions. We rely primarily on the concept of \emph{commitments}, commonly used in cryptography to establish \emph{precedence}.
Specifically, we revisit the discussion of data commitment and timestamping made by \jia. The analysis we conducted in this paper shows how certain choices made by \jia in that regard need to be revisited. %

\vspace{2mm}
\noindent{\bf Data Commitment:} 
\label{sssec:data_commitment} \jia require the prover to create signatures of the data used at each step. They do so to prevent the prover from {\em denying usage} of specific data segments at a later time (\ie serve as a cryptographic commitment). Such signatures also help circumvent spoofs that can arise from the non-uniqueness of a gradient-based update (\ie multiple data segments can potentially result in the same gradient update~\cite{thudi2021necessity,DBLP:journals/corr/abs-2104-09667}). While \jia advocate for public data release. We argue this is unnecessary, and may result in privacy violations when the data used to train the model is proprietary/sensitive, as cryptographic commitment schemes binds the prover to the data segment used at a particular step. The prover can share the data at the time of verification through a secure channel. See Figure~\ref{fig:data_commit} in Appendix~\ref{app:additional_figures} for the flow of communication between the prover and the verifier.

From our experiments (not presented here for the sake of brevity), we observe that training data is essential for the attacks by \zhang. While it is conceivable that one can synthesize data with the aforementioned property (\ie success in the context of synthetic adversarial update strategies), more analysis is required to understand the {\em cost} associated with data synthesis. Data commitment schemes force adversaries to create {\em all data} before submitting the spoof, and not only at verification time (and for steps being verified).

\vspace{2mm}
\noindent{\bf Timestamping:} \label{sssec:timestamping} 
We emphasize the need of timestamping the proof or its signature upon submission (or publishing it in a public ledger) as introduced in the original \pol protocol~\cite{jia2021proof}. This will prevent (a) replay attacks where the adversary submits the exact same \pol as the victim trainer; and (b) any attacks that involve the adversary using the exact model parameters of the victim model as the final state of their spoof.

\section{Conclusion}
\label{sec:conclusion}

Given the  open problems we concluded Section~\ref{sec:giving_guarantees_with_pol} and Section~\ref{sec:theory_cheap_spoofing} with, we now revisit the question we posted at the beginning of the paper: can \pol be robust? 
As it stands, the answer is clear: formally proving the robustness of a proof verification mechanism for PoL is not currently possible. 

One possible solution to circumvent these fundamental limits in ML theory and our understanding of optimization, is to rely more on cryptography. Although one of the motivations for approaches like \pol is to avoid using cryptography due to its limited scalability when it comes to training deep neural networks, it does not preclude us from envisioning that cryptographic primitives may be combined with ideas from \pol to provide analytical security guarantees. For example, as we have already discussed, data commitment mechanisms may be extended to other parts of the \pol protocol to reduce the attack surface of \pol.

Moving forward, we laid down generic properties that spoofing adversaries must satisfy.  We believe future work can expand on these results to prove if such adversaries can or cannot exist thereby answering one of the open problems towards formally guaranteeing the robustness of \pol. Similarly, future work can investigate how to instantiate the optimal verification strategy we showed exists, and whether adversaries can bypass this strategy.

\section*{Acknowledgements}
We would like to acknowledge our sponsors, who support our research with financial and in-kind contributions: Amazon, Apple, CIFAR
through the Canada CIFAR AI Chair, DARPA through the GARD project, Intel, Meta, NFRF through an Exploration
grant, NSERC through the COHESA Strategic Alliance, the Ontario Early Researcher Award, and the Sloan Foundation. Resources used in preparing this research were provided,
in part, by the Province of Ontario, the Government of Canada through CIFAR, and companies sponsoring the Vector
Institute.
We would also like to thank CleverHans lab group members for their feedback.

\newpage
{
\printbibliography
}

\appendices

\section{Table of Notations and Terminology}
\label{app:notations}

\begin{table*}[ht]
\normalsize
\centering
\begin{tabularx}{\linewidth}{c X}
\toprule
{\bf Variable} & {\bf Purpose} \\
\midrule
\midrule
$\prover$ & prover \\
$\adv$ & adversary \\
$\mathbf{W}$ & weight space of models \\
$k$ & the checkpointing interval \\ 
$s$ & number of steps/batches per epoch \\ 
$W_t$ & the $t$-th checkpoint in the honest prover's training process\\ 
$W'_t$ & the reproduced model weights by the verifier for model weights $W_t$ in the \pol\\ 
$\hat{W}_t$ & $t$-th checkpoint generated by adversary's spoofing process \\ 
$\hat{W}'_t$ & the reproduced model weights by the verifier for model weights $\hat{W}_t$ in the \pol\\ 
$W_T$ & the final weights of the victim model \\
$\hat{W}_{\hat{T}}$ & the final weights of the adversary's model \\
$g_t$ &  the update from step $t$ to $t+k$ in the prover's \pol (\ie $W_{t+k} - W_t $) \\
$g'_t$ & the update from step $t$ in the prover's \pol to step $t+k$ reproduced by the verifier (\ie $W'_{t+k} - W_t $)\\ 
$\hat{g}_t$ &  the update from step $t$ to $t+k$ in the adversary's \pol (\ie $\hat{W}_{t+k} - \hat{W}_t $)\\
$\hat{g}'_t$ &  the update from step $t$ in the adversary's \pol to step $t+k$ reproduced by the verifier (\ie $\hat{W}'_{t+k} - \hat{W}_t $) \\ 
$D$, $D_t$ & the honest prover's dataset; if with subscript $t$,  then it represents the batch of data used in the $t^{th}$ training step\\
$\hat{D}$, $\hat{D}_t$ & the adversary's dataset; if with subscript $t$,  then it represents the batch of data used in the $t^{th}$ training step \\
$M_t$ & the honest prover's training metadata (\eg hyperparameters) at the $t^{th}$ training step\\
$\hat{M}_t$ & the adversary's training metadata (\eg hyperparameters) at the $t^{th}$ training step \\
$\delta$, $\delta_t$ & the threshold to bound the step-wise noise for verification as described in the original Proof-of-Learning (\pol) algorithm; if with subscript $t$, then it is specific to the $t^{th}$ training step \\
$Q$ & number of updates the verifier will verify per epoch ($Q$ in the Top-$Q$ mechanism) \\
$\eta$ & the learning rate \\
$d(\cdot)$ & some distance metric \\
$\alpha$  & scaling factor for Lemma~\ref{lem:adaptive_threshold} \\
$||\cdot||$  & norm of a given vector, if not otherwise specified, then it is $\ell_2$ norm \\
$\tau$ & a specific value of true positive rate (TPR) \\
$\mathcal{U}(\cdots)$ & function of model weights, training data, and training metadata that updates the model weights \\
$C$ & cost function \\
$(\mathbf{g},d,\boldsymbol{\delta})$-proofs & proofs using update rules $\mathbf{g}_i \in \mathbf{g}$ passing thresholds $\boldsymbol{\delta}$ in metric $d$\\
$A_{D,W_T}$ & set of $(\mathbf{g},d,\boldsymbol{\delta})$-proofs ending in $W_T$ generated by a specific dataset $D$ \\
$\mathcal{P}$ & a proof generated by \pol\\
$F: \cdots \rightarrow \cdots$ & algorithm for creating proof \eg Algorithm 1 of Jia~\etal~\cite{jia2021proof} \\
$\mathbb{S}$ & subset of model updates \\
$B_{r}(v)$ & balls centered at vector $v$ of radius $r$ \\
\bottomrule
\end{tabularx}
\caption{Notations}
\label{tab:notation_attack}
\end{table*}

\begin{table*}[ht]
\normalsize
\begin{tabularx}{\linewidth}{l X}
\toprule
{\bf Spoofing} & {\bf Definition} \\
\midrule
\midrule
Retraining-based Spoofing & The adversary aims to spoof the exact proof of an honest trainer ending at $W_T$; Jia~\etal~\cite{jia2021proof} showed it is more computationally expensive even if the adversary knows the data points used at each training step\\
Structurally Correct Spoofing & The adversary aims to spoof a proof ending at $W_T$ which contains at least 1 invalid update, \ie update that cannot be produced by honest training, but still passes the verification; it is closely related to \pol's role on efficient verification as described in Section~\ref{sec:giving_guarantees_with_pol}, and corresponding examples can be found in Section~\ref{sec:fundamentals_revisited}\\
Stochastic Spoofing & The adversary aims to spoof a valid proof ending at $W_T$ with less cost than training; the precedence role of \pol described in Section~\ref{ssec:stochastic_spoofing_experiment} relies on  the non-existence of such spoofing with less cost than training, and empirical results are in Section~\ref{ssec:stochastic_spoofing_experiment}\\
Distillation-based Spoofing & The adversary aims to spoof a valid proof ending at a functionally similar state to $W_T$ (\eg model extraction). We do not discuss this type of spoofing in this paper \\
\bottomrule
\end{tabularx}
\caption{ \textbf{Types of spoofing:} listed are the types of spoofing against \pol categorized by Jia~\etal~\cite{jia2021proof} along with pointers to where we discuss them or why they are not discussed.}
\label{tab:spoofing}
\end{table*}

We defined the notations in Table~\ref{tab:notation_attack}, and the categorization of spoofing in Table~\ref{tab:spoofing}.
\section{Computational Cost Analysis for the Attacks}
\label{app:cost}
Here we provide a detailed analysis of the computational cost of the attacks mentioned in the paper. We use the cost of 1 forward propagation (FP) as the basic unit, which is approximately equal to N floating point operations (where N is the number of model parameters). Note that 1 back propagation costs approximately 2 FPs, and adding parameters of two model states together costs approximately 1 FP. Other notation used in this appendix includes: number of iterations of updating the adversarial example (n), and number of batches per epoch(s). All attacks first linearly interpolate between a random initialized state and the final stolen model state so they can spoof with the same length, so w.l.o.g. we may compare their cost step-wise (\ie from state $t$ to state $t+k$).  

\vspace{2mm}
\noindent \textbf{Infinitesimal Update Attack:} Apart from linear interpolation, the Infinitesimal Update attack does not require any other computation, and the linear interpolation is done once for every model parameter so the step-wise cost is 1 FP.

\vspace{2mm}
\noindent \textbf{Attack by Zhang~\etal~\cite{zhang2021adversarial}:} Attack 2 by Zhang~\etal needs to interpolate for every single update between state $t$ and state $t+k$ (e.g., $t$ to $t+1$, $t+1$ to $t+2$, …), so 1 FP is required for every update (\ie $k$ FPs in total for linear interpolation).  Besides, for each of these updates, 1 FP and 1 backward propagation (= 3 FPs) are needed to compute the gradient of the model. Then another backward propagation is required to differentiate the norm of the gradient with respect to the inputs to the model, this is essentially a second order gradient and the cost depends on the algorithm used to compute it. By measuring time of the code released by Zhang~\etal, we found empirically it takes more than 20 times than the gradient computation, so 40 FPs. Adding all these together, 43 FPs is needed for a single iteration of creating the adversarial examples, so Attack 2 costs at least $(43 \cdot n + 1) \cdot k$ FPs per step.

Zhang~\etal tried to parallelize their Attack 2, which resulted in a different attack (their Attack 3), but it would still cost $43 \cdot n + 1$ FPs per step.

\vspace{2mm}
\noindent \textbf{Blindfold top-$Q$ Attack:} There are two cases here: (a) if it is one of the top-$Q$ updates, then $k$ valid gradient updates (3 FPs) need to be computed, so the cost is $3 \dot k + 1$ FPs (1 comes from linear interpolation); (b) if it is not a top-$Q$ update, then nothing besides linear interpolation needs to be done, so the cost is 1 FP. In every epoch, there is $Q$ top-$Q$ updates and $s-Q$ non-top-$Q$ updates, so the expected step-wise cost is $(Q/s) \cdot (3 \cdot k + 1) + (s-Q)/s = (3 \cdot k \cdot Q) / s + 1$ FP.
\section{Adversarial Reconstruction of Known Model Weights by Adversarial Update Rules. }
\label{app:rna_attack}
\vspace{1mm}
\noindent{\bf Designing Adversarial Update Rules:} From~{Section}~\ref{sssec:adv_reconstruct}, we found that it's hard to successfully create a stochastic spoof even when the original training data is given to the adversary. Herein, as a worst-case scenario in addition to access of training data, we consider that the adversaries can manipulate the internal states of the optimizer to use adversarial update rules. Note that the ability for an adversary to control the optimizer is not part of the threat model of \pol. We are only exploring to see if certain optimizers/update rules can help converge to a specific final weight $W_T$ faster.
{The adversarial update rule we considered here} is inspired by an existing optimizer that exploits information about the training paths for DNNs, Regularized Non-linear Acceleration (RNA)~\cite{scieur2019regularized}. RNA is a convergence acceleration technique for generic optimization problems. It extrapolates the trajectory path history for iterative optimization problems to improve the convergence. Inspired by RNA, we designed an adversarial update rule that uses a linear combination of intermediate weights, or individual gradient directions, to reach a next weight also minimizes the distance to the victim model's weights (depicted in Figure~\ref{fig:rna_intuition}). For each RNA round, the adversary regularly trains for a few steps and records the intermediate weights $\hat{W}_i$'s (and their updates $\hat{g}_i = \hat{W}_{i+1} - \hat{W}_i$). The adversary then solves for coefficients such that $\hat{c}^*_i = \argmin_{\hat{c}_i} || W_T - \sum_{i} \hat{c}_i \cdot \hat{g}_i ||$ and obtains the weights for the next step as $\sum_{i} \hat{c}^*_i \cdot \hat{g}_i$ to minimize its distance to the victim model, $W_T$. In honest RNA training, the extrapolation coefficients are calculated based on the trajectory information. In the adversarial update rule, the coefficients are spoofed and the adversary argues the hyperparameters were selected by this customized update rule. We call this adversarial update rule the {\em \rnaattack}. 

We empirically evaluate the effectiveness of \rnaattack.
We perform 10 RNA steps per epoch and the results are shown in Figure~\ref{fig:rna_results}. We measure the distance between the adversary's weights generated by the RNA attack and $W_T$. As we can see from the figure, as opposed to the previous two attacks we discussed, when using the RNA attack the distance between the adversary's weights and $W_T$ consistently decreases. However, after 
{the same computational cost as honest training (200 epochs),}
the distance remains significantly larger than $0$. Thus, the adversary cannot generate a valid spoof.

\begin{figure*}[ht]
\centering
\subfloat[Intuition of \rnaattack]{\includegraphics[width=0.5\linewidth]{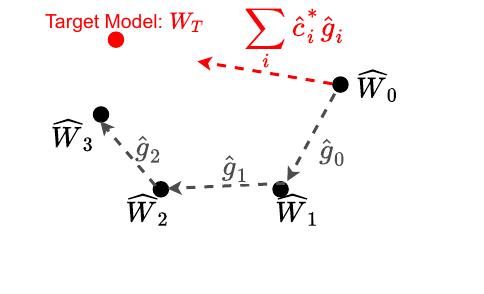}\label{fig:rna_intuition}}
\subfloat[Effectiveness of \rnaattack]{\includegraphics[width=0.45\linewidth]{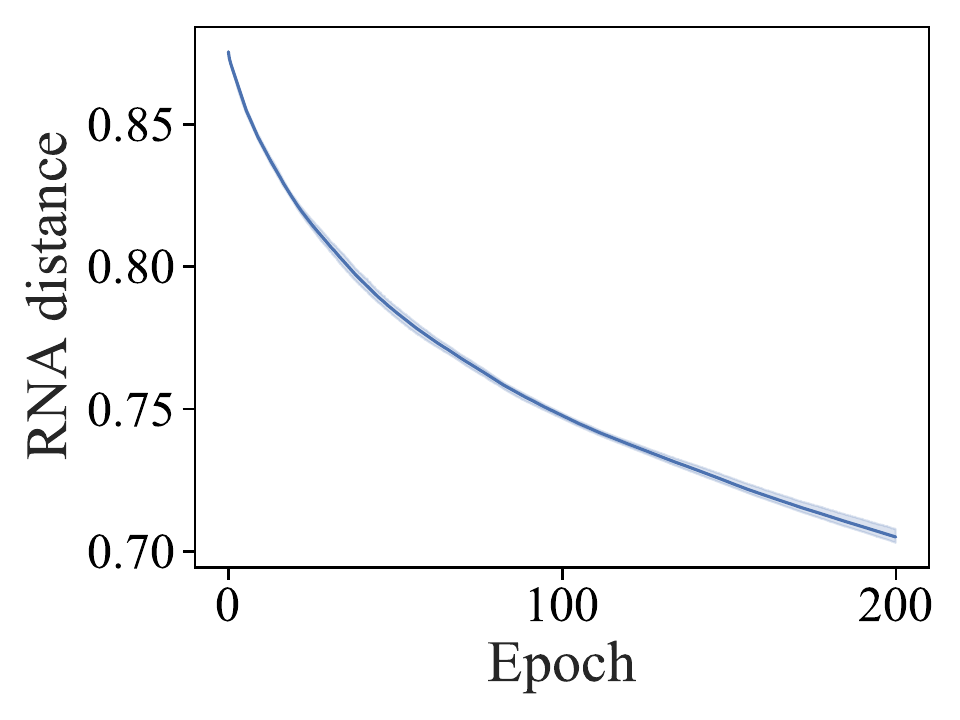}\label{fig:rna_results}}
\caption{ \textbf{Illustrating the \rnaattack and its Effectiveness.} (a) For each step of the attack, the RNA adversary solves for the optimal coefficients $\hat{c}^{*}_i$ that minimize the distance between the victim model and a linear combination of updates from honest training steps. (b) The distance between the victim model and the adversary's model generated using the \rnaattack is plotted against the total number of epochs for honest training. We followed the same setup as defined at the beginning of Section~\ref{sec:fundamentals_revisited} using ResNet-20 on CIFAR-10 dataset.  We repeat this experiment five times to get confidence intervals, which are too small to see. The distance is decreasing but still significantly larger than 0 after 200 epochs (honest training cost). Thus, a spoof cannot be generated.}

\label{fig:rna}
\end{figure*}
\section{Novelty compared to previous attacks by \zhang}
\label{app:novelty}

\noindent {\bf Threat Model}: Previous attacks by \zhang assumed that the adversary has knowledge of  (1) the training data, (2) the parameters used for creating the proof \eg the checkpointing interval $k$, and (3) the verification threshold $\delta$). Regarding (2), the adversary is also assumed to have the capability to modify the checkpointing interval $k$, which is unrealistic: this value should be set by the verifier. Regarding (3), we note that this assumption violates the threat model initially stated in Jia et al.~\cite{jia2021proof}: indeed, there is no reason for the verifier to reveal the parameters used for verification to the provers and adversaries, yet alone let them set the values of these parameters. Overall, these three assumptions make for an unrealistic threat model, which at times violates the threat model initially stated by Jia et al.~\cite{jia2021proof}. Instead, we introduce in our work attacks that can be mounted within a simpler threat model. 

\noindent {\bf Customized Attacks}: The attack by \zhang are tailored to the CIFAR10 dataset: they were found to perform significantly worse on CIFAR100 despite extensive hyperparameter tuning. Instead, we propose attacks that are independent of the training setup (such as the model architecture and dataset). Our adversaries only assume knowledge of  the parameters of the model which they aim to spoof the proof for.

\noindent {\bf Formal Analysis}: While prior work by \zhang identified limitations of \pol \textit{empirically}, in a threat model whose limitations we identified above, we are the first to \textit{systematically} study the assumptions needed to prove PoL’s robustness. We accordingly taxonomize the vulnerabilities and corresponding attacks along the assumptions being exploited: (a) reproducibility, and (b) representativeness (refer Section~\ref{sec:giving_guarantees_with_pol}). We further show that validating these assumptions reduces to solving fundamental open questions in deep learning theory. Out of the two types of vulnerabilities, \zhang only focused on empirically exploiting the vulnerability caused by the reproducibility assumption. Some computationally expensive components of their attacks are unnecessary for exploiting this. Furthermore, they did not account for the vulnerability caused by the representativeness assumption. Therefore, our analysis is also the first attempt to explore this previously unexamined aspect of PoL vulnerability.

\noindent {\bf Attack Performance}: Our attack is more effective because it exploits the aforementioned vulnerabilities we uncovered. To create an update whose magnitude is lesser than a certain threshold, \zhang’s attack relies on expensive optimization procedures. Instead, our attack uses a small learning rate in the proof it produces to achieve the same effect in an inexpensive manner. Apart from being independent of the training setup, (a) we achieve more than 40x speedup in all scenarios since our attack does not require any form of training, and (b) the normalized reproduction error of our attack is consistently smaller on CIFAR10, and even outperforms \zhang’s attack by more than 10x on CIFAR100. %
In addition, we also proposed a novel attack exploiting  the assumption that the correctness of larger training updates (in terms of magnitude) may imply the correctness of the entire training. This attack is computationally efficient and exploits a newly discovered vulnerability not discussed by previous work.

\section{Additional Figures}
\label{app:additional_figures}

Figure~\ref{fig:data_commit} illustrates the communication between the verifier and the prover described in Section~\ref{sssec:data_commitment}. 
Figures~\ref{fig:defense_intuition_adp_delta} and~\ref{fig:adaptive_delta_all} demonstrate the intuition for Lemma~\ref{lem:adaptive_threshold} and its empirical validation respectively. 
Figures~\ref{fig:boa_cifar10_leftovers}, and~\ref{fig:boa_cifar100} are additional evaluation results for experiments in~\ref{sssec:adv_reconstruct}.

\begin{figure}[ht!]
    \centering
    \includegraphics[width=\linewidth]{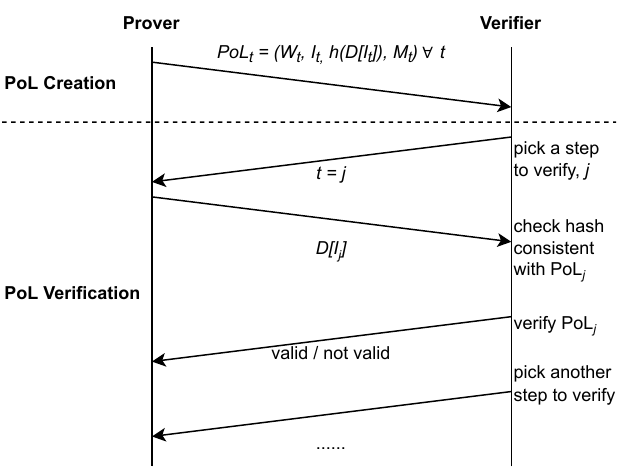}
    \caption{\textbf{Overview of the communication/data exchange between the verifier and the prover:} during proof creation, model checkpoints, training metadata, indices and hashes of training data are submitted to the verifier. During proof verification, the prover sends training data at step $j$ to the verifier, who then checks their hashes are consistent with the \pol. The verification is done by reproducing the model at step $j+1$ and verifying it is similar to the checkpoint at $j+1$ in the \pol.}
    \label{fig:data_commit}
\end{figure}

\begin{figure*}[t]
\centering
\subfloat[Valid Updates (with small magnitudes)]{
\includegraphics[height=4cm]{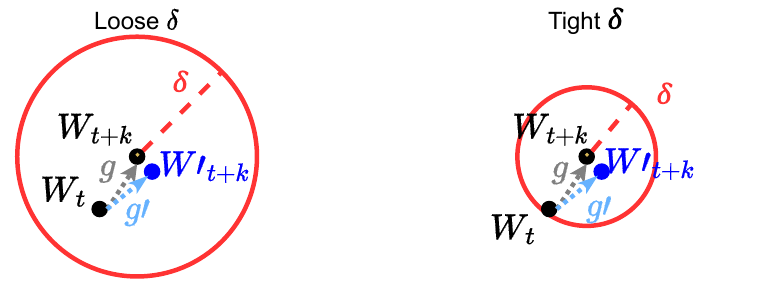}
\label{fig:adp_delta_intuition_legit}
}
\vspace{0.5mm}
\subfloat[Spoofed Updates]
{
\includegraphics[height=4cm]{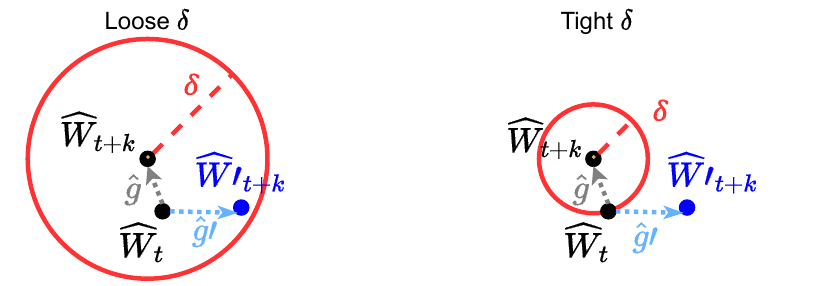}
\label{fig:adp_delta_intuition_spoof}
}
\caption{ Intuition for Lemma~\ref{lem:adaptive_threshold} for (a) valid updates $g$ and (b) spoofed updates $\hat{g}$ (where $g', \hat{g}'$ denote their verifier-reproduced versions). We visualize the setting where the norm of the update or reproduced update is less than the noise tolerance $\delta$. For a valid update, the discrepancy between $g_t$ and $g'_t$ is small because it is caused by noise. Therefore, scaling the $\delta =\alpha \cdot \min(||\hat{g}||, ||\hat{g}'||) $ (here we visualize using $\alpha=1$) following Lemma~\ref{lem:adaptive_threshold} should not make their discrepancy exceed the tolerance. However, in a spoofed update with near-zero updates, $\hat{g}_t$ and $\hat{g}'_t$ exceed this tighter noise threshold and fail verification.
}
\label{fig:defense_intuition_adp_delta}
\end{figure*}

\begin{figure*}[ht!]
    \centering
    \subfloat[Infinitesimal Update Attack \\ (CIFAR-10)
\label{fig:adaptive_delta_small_lr}]
{
\includegraphics[width=0.4\linewidth]{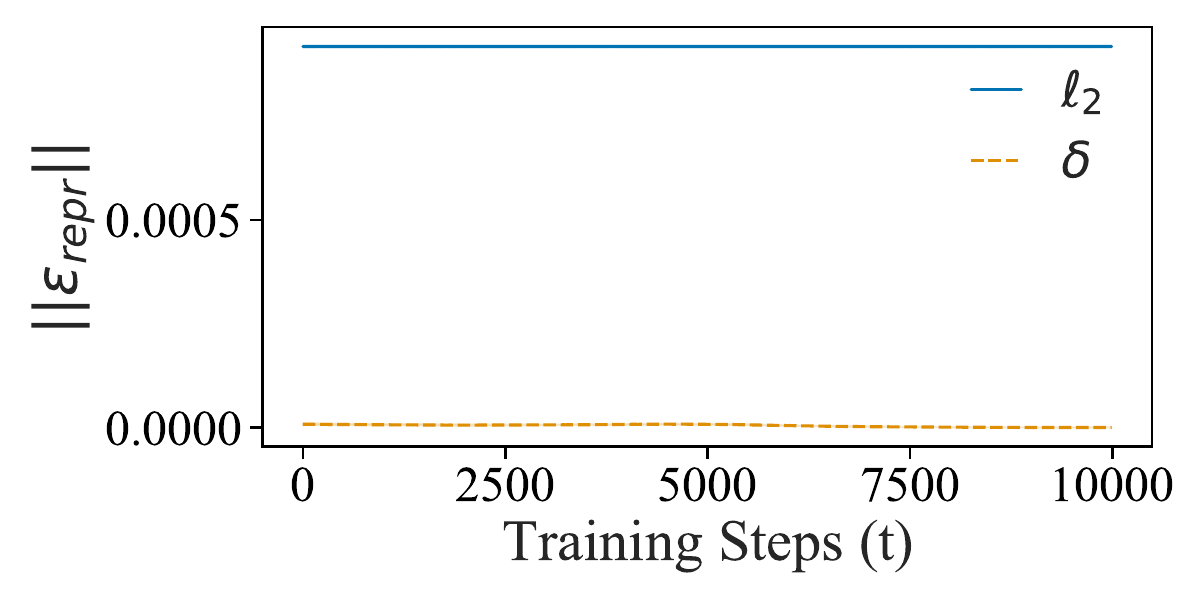}
}
\subfloat[Attack 2 by Zhang~\etal~\cite{zhang2021adversarial} \\ (CIFAR-10)
\label{fig:adaptive_delta_normal}]{
\includegraphics[width=0.4\linewidth]{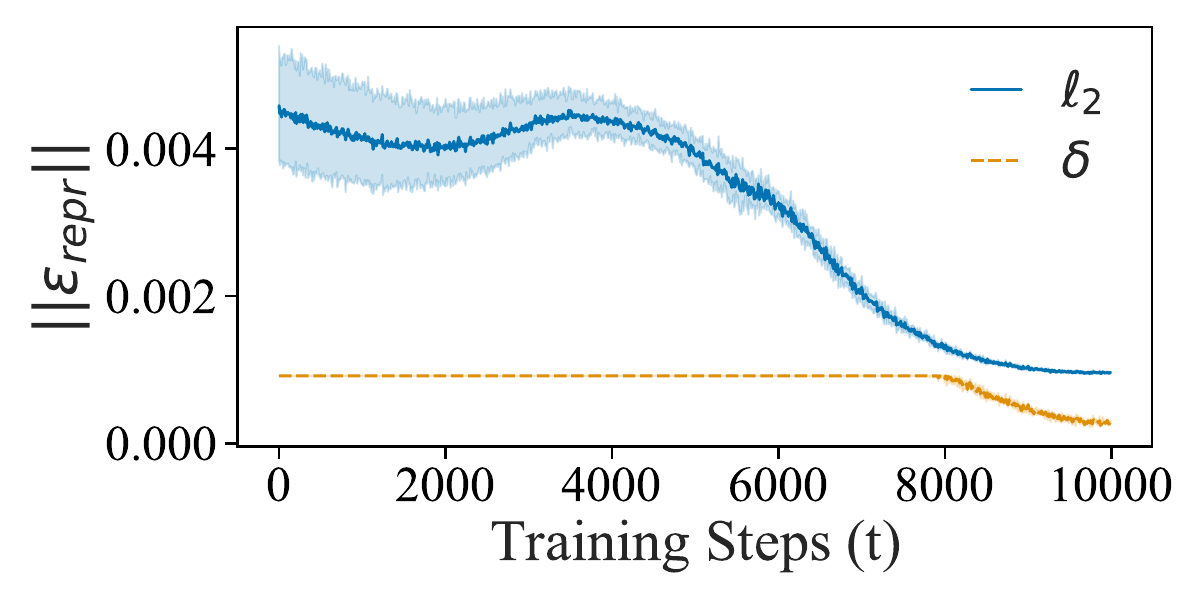}
}
\\
\subfloat[Infinitesimal Update Attack \\ (CIFAR-100)]
{
\includegraphics[width=0.4\linewidth]{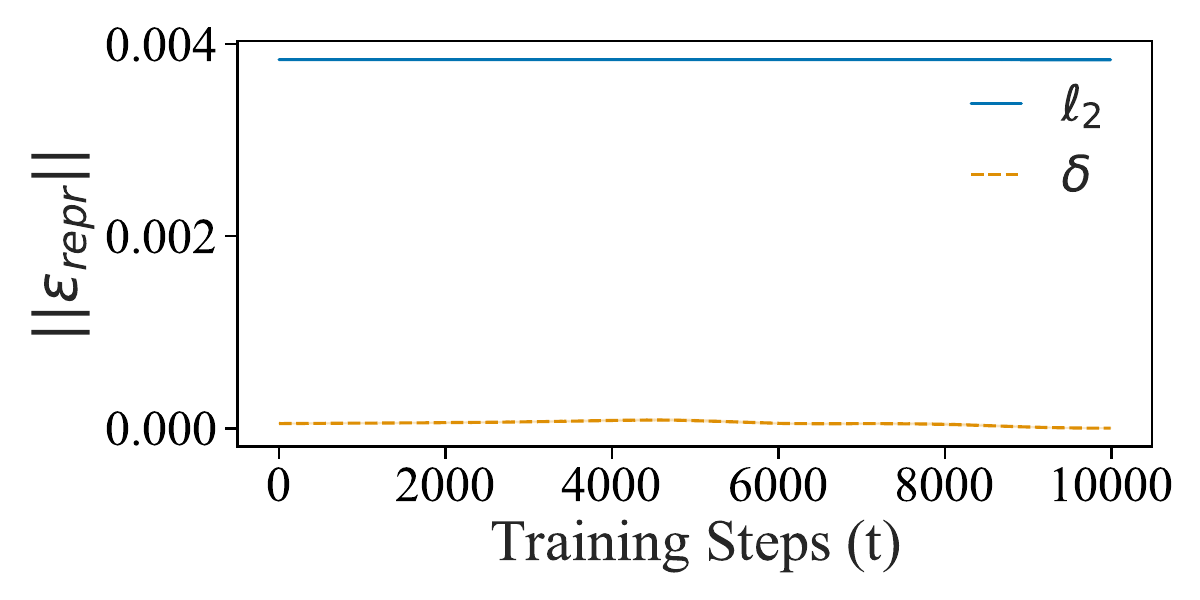}
}
\subfloat[Attack 2 by Zhang~\etal~\cite{zhang2021adversarial} \\ (CIFAR-100)]
{
\includegraphics[width=0.4\linewidth]{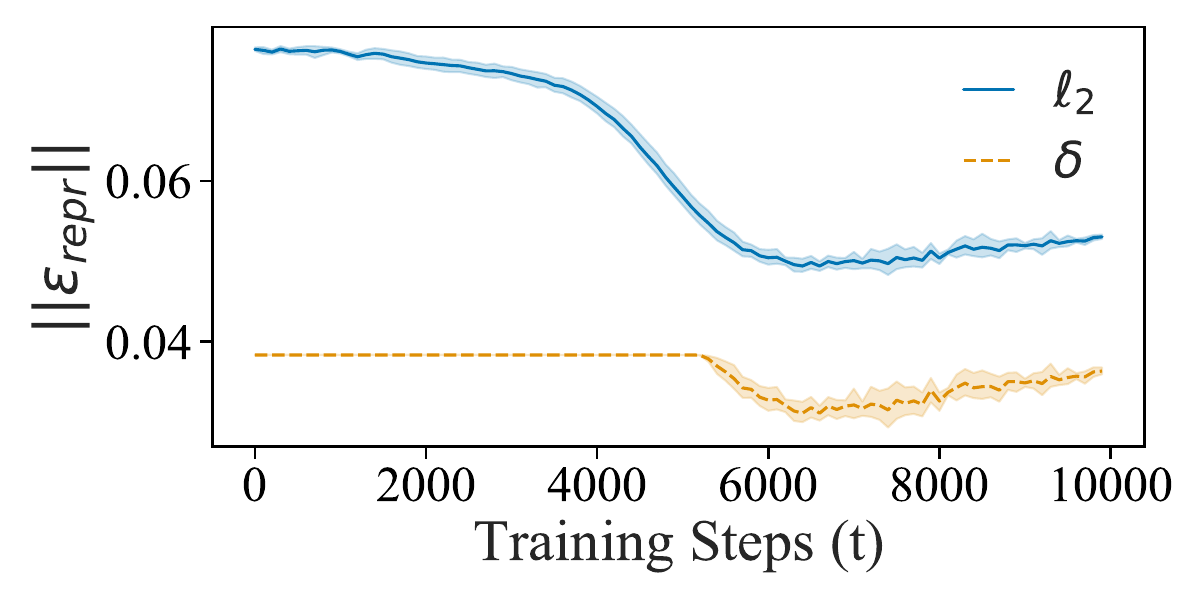}
}
\label{fig:adaptive_delta}
    \caption{ 
    \textbf{Empirical validation of Lemma~\ref{lem:adaptive_threshold}:} The same setup as Figure~\ref{fig:baseline} is used except the threshold is set adaptively following Lemma~\ref{lem:adaptive_threshold} (with $\alpha=1$). Observe that now the \nre is consistently greater than $\delta$ (\ie both known spoofs would be rejected). To confirm the intuition shown in Figure~\ref{fig:adp_delta_intuition_legit} that Lemma~\ref{lem:adaptive_threshold} does not incur false rejections, we empirically verify that a model trained with $100$x smaller learning rate consistently passes verification (across $5$ runs).
    }
    \label{fig:adaptive_delta_all}
\end{figure*}

\begin{figure*}[ht!]
    \centering
    \subfloat[Step 0]{\includegraphics[width=0.3\linewidth]{figures/hist_CIFAR10_resnet20_0_1_dist_2_long.pdf}}
    \subfloat[Step 10000]{\includegraphics[width=0.3\linewidth]{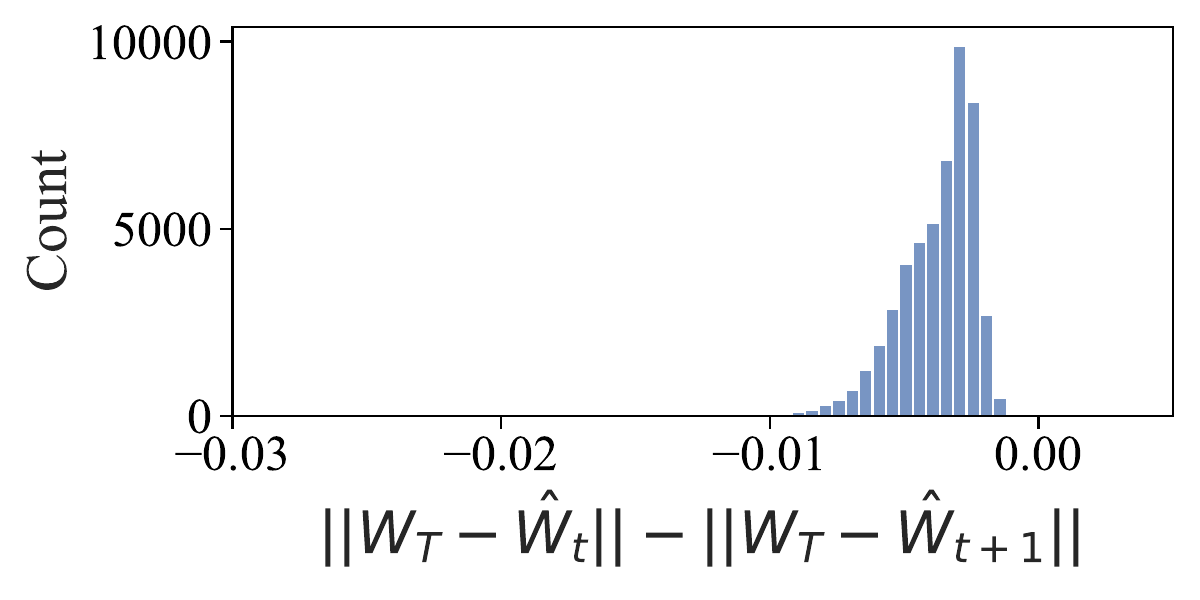}}
    \subfloat[Step 20000]{\includegraphics[width=0.3\linewidth]{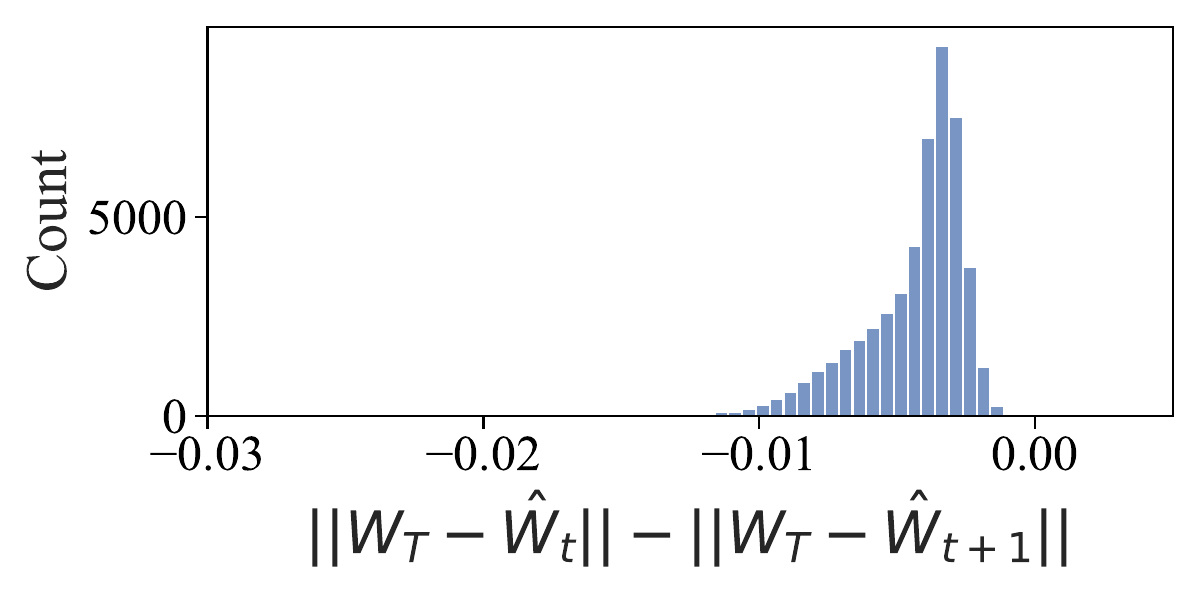}} 
    
    \subfloat[Step 30000]{\includegraphics[width=0.3\linewidth]{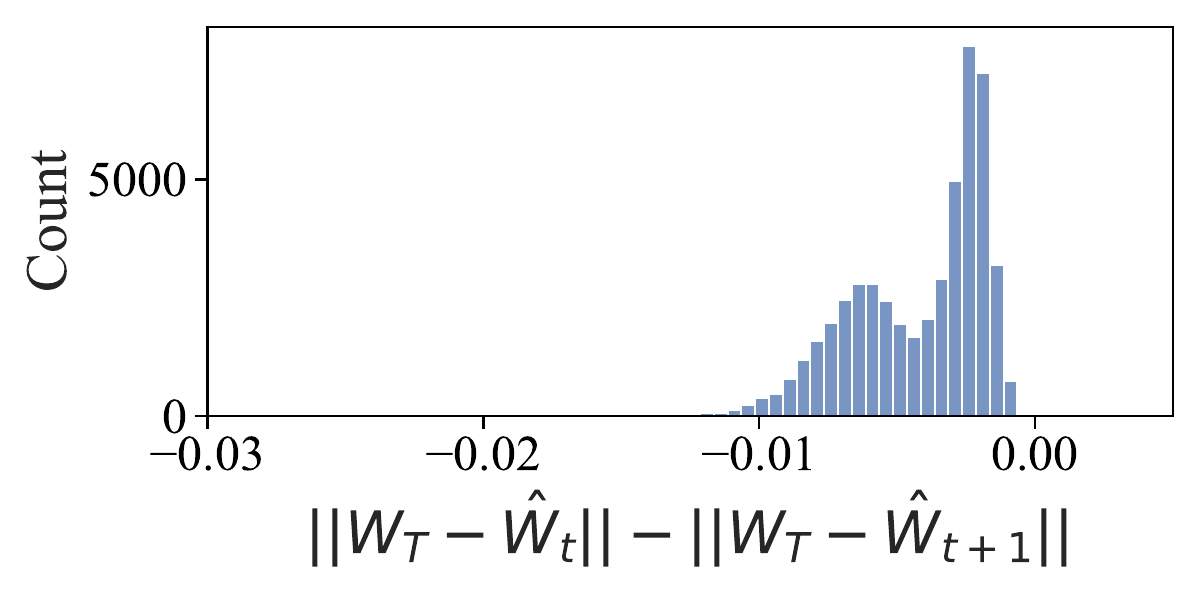}}
    \subfloat[Step 40000]{\includegraphics[width=0.3\linewidth]{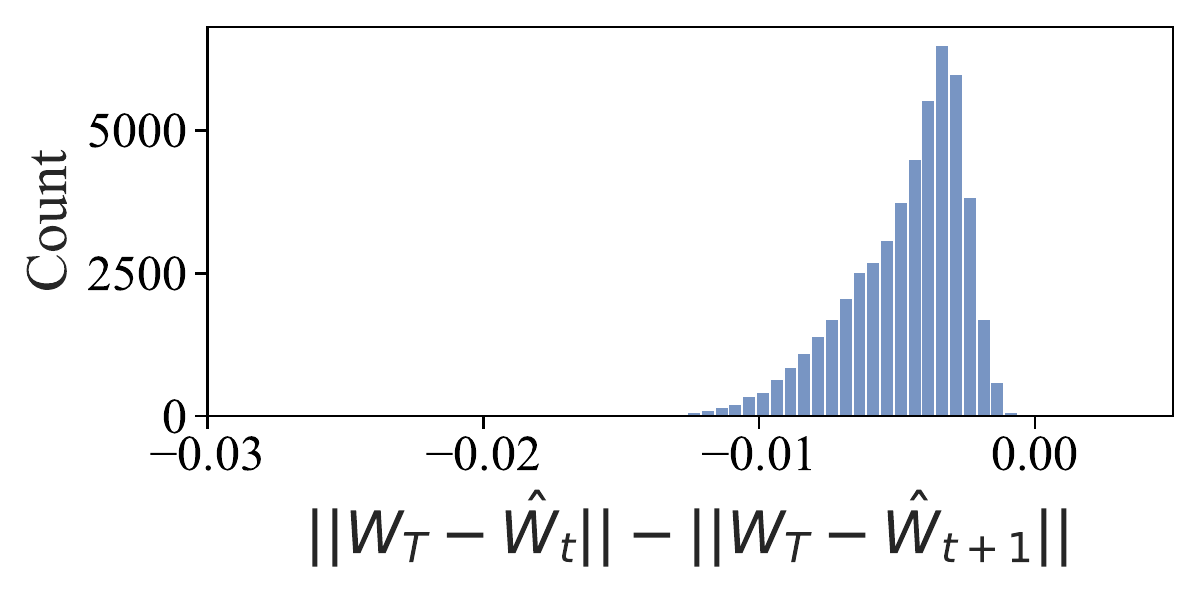}}
    \subfloat[Step 50000]{\includegraphics[width=0.3\linewidth]{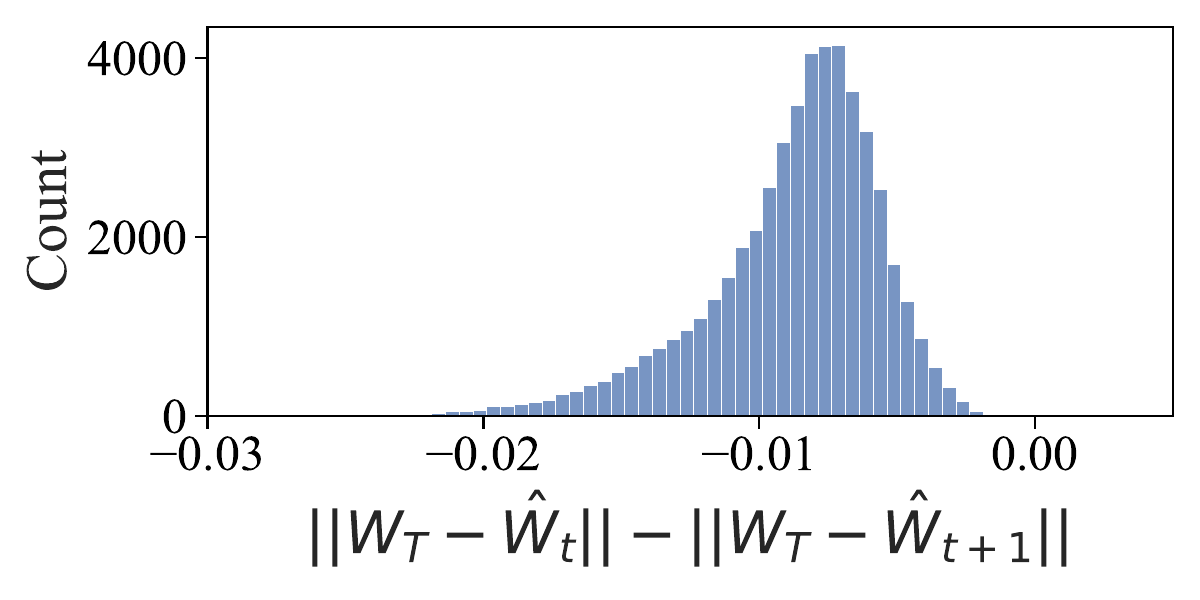}}
    
    \subfloat[Step 60000]{\includegraphics[width=0.3\linewidth]{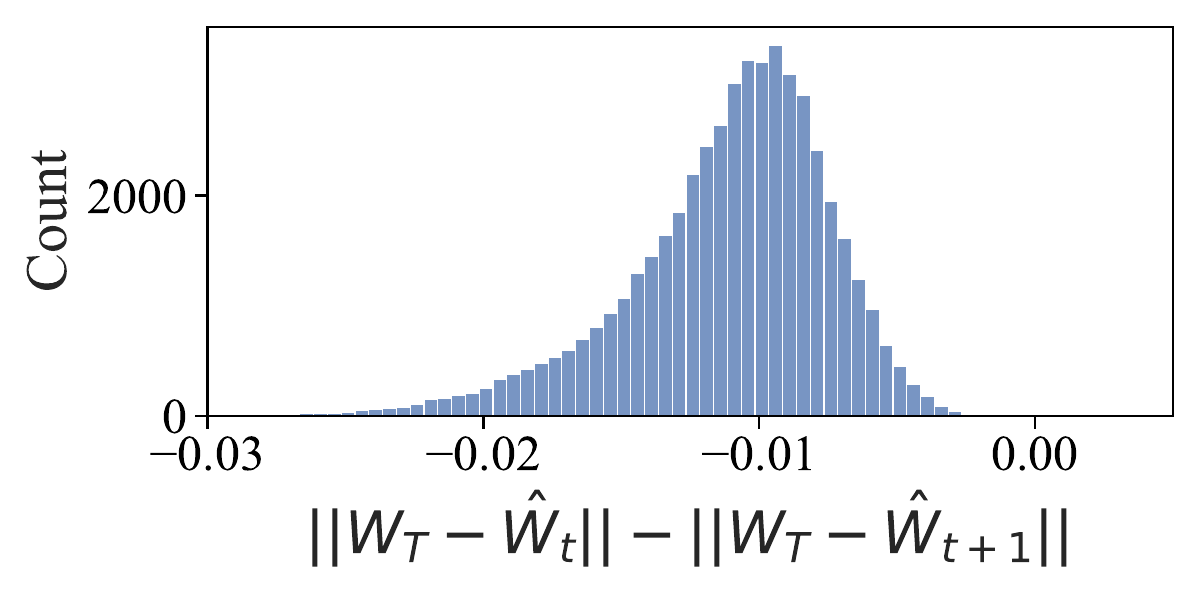}}
    \subfloat[Step 70000]{\includegraphics[width=0.3\linewidth]{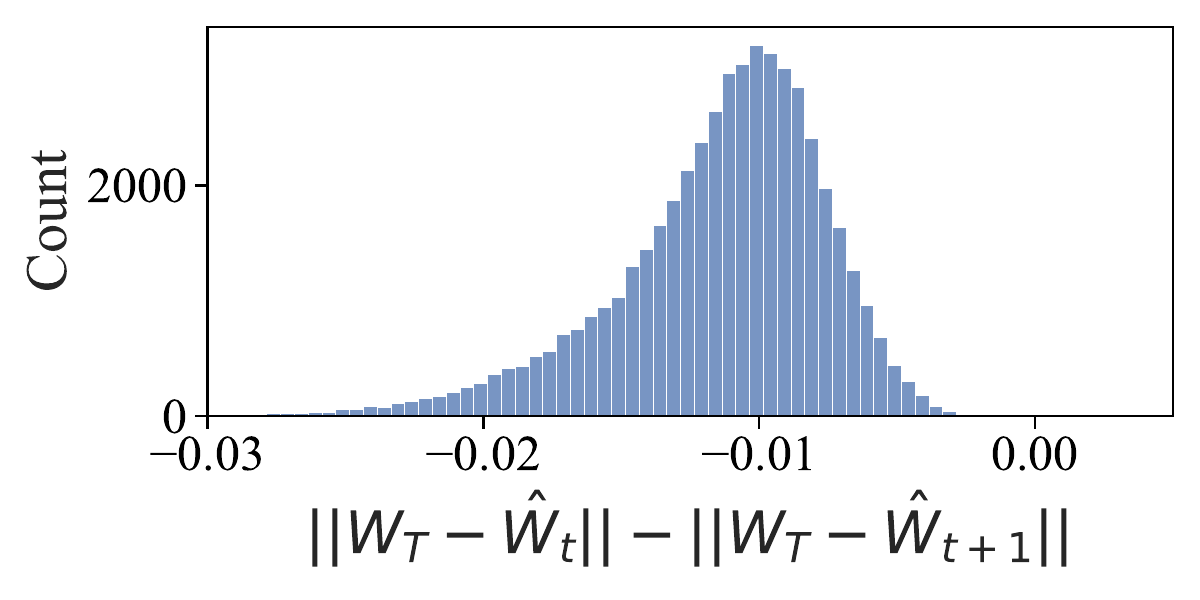}}
     \subfloat[Training Finished (Epoch 200)]{\includegraphics[width=0.3\linewidth]{figures/hist_CIFAR10_resnet20_78125_1_dist_2_long.pdf}}
    \caption{ 
    \textbf{No individual update pushes the adversary's model $\hat{W}_t$ closer to the prover's final model $W_T$.} These are plotted in the same setting as Figure~\ref{fig:hist} to illustrate finer grained information on the evolution of the updates throughout training. Note that all gradients at all intermediate steps in training push the adversary's model away from $W_T$, as indicated by the negative values on the $x$-axis. We use a ResNet-20 on the CIFAR-10 dataset.
    }
    \label{fig:boa_cifar10_leftovers}
\end{figure*}

\begin{figure*}[ht!]
    \centering
    \subfloat[Step 0]{\includegraphics[width=0.3\linewidth]{figures/hist_CIFAR100_resnet50_0_100_dist_2_long.pdf}}
    \subfloat[Step 10000]{\includegraphics[width=0.3\linewidth]{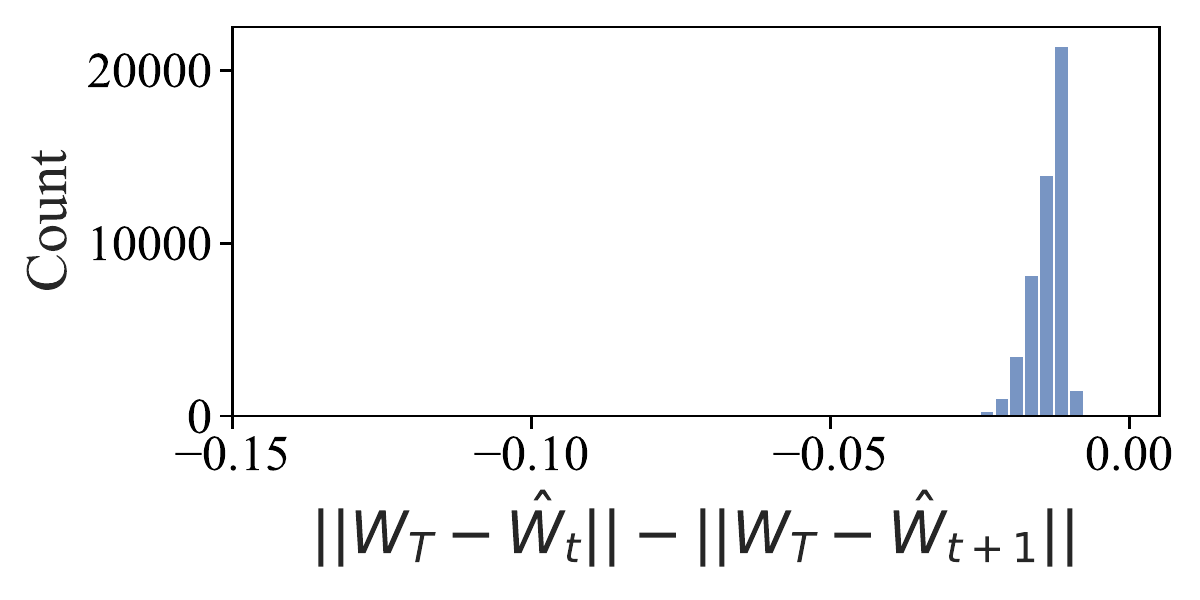}} 
    \subfloat[Step 20000]{\includegraphics[width=0.3\linewidth]{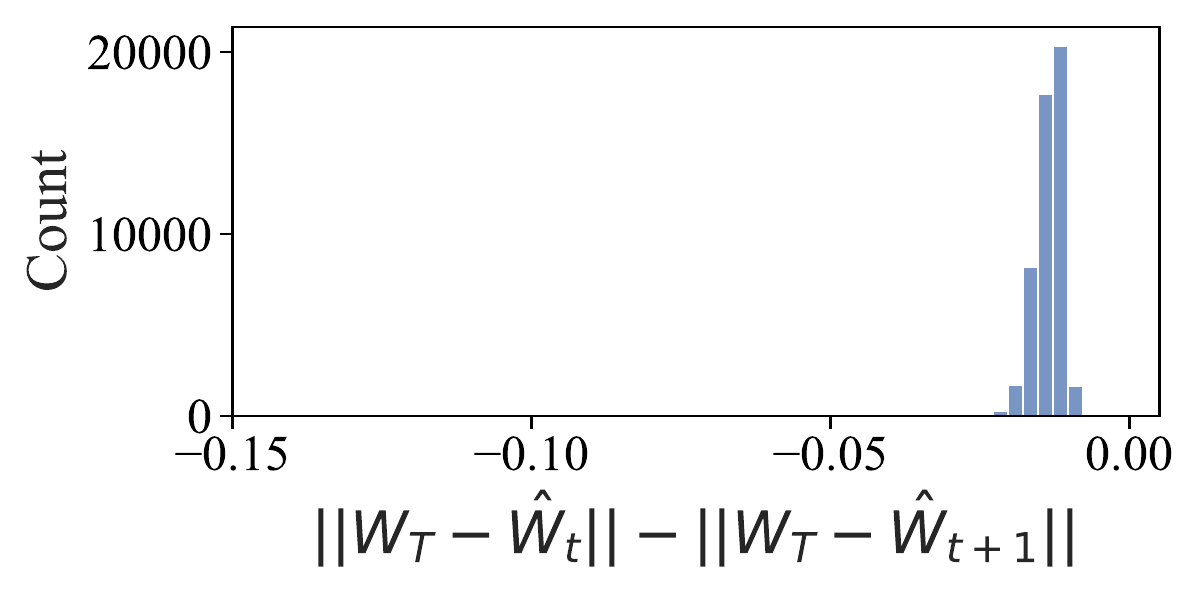}}
    
    \subfloat[Step 30000]{\includegraphics[width=0.3\linewidth]{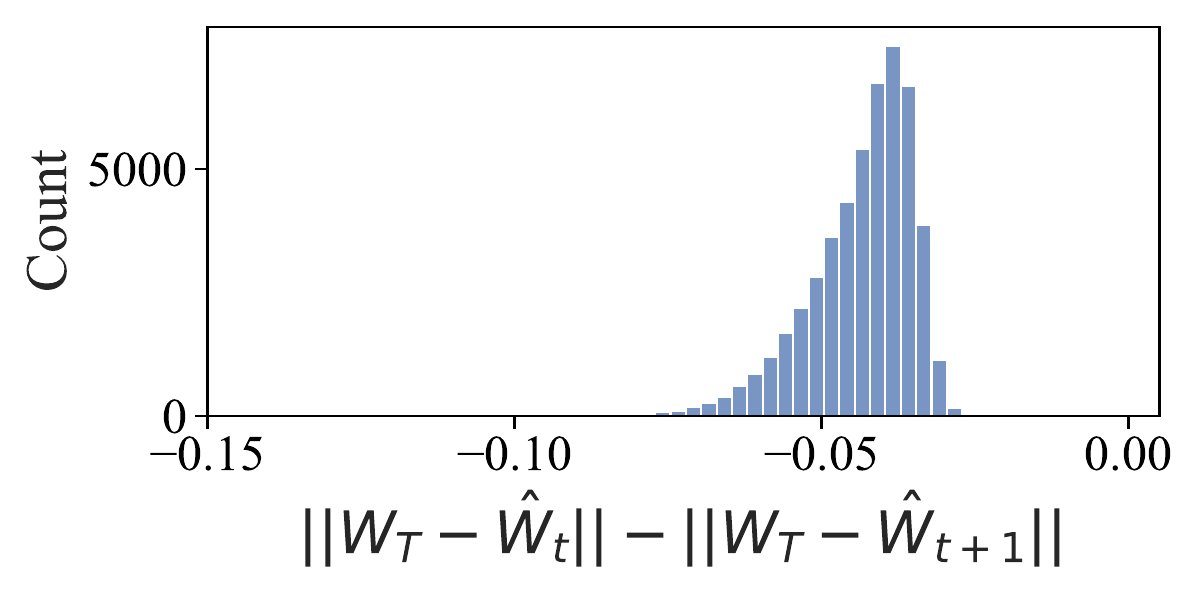}} 
    \subfloat[Step 40000]{\includegraphics[width=0.3\linewidth]{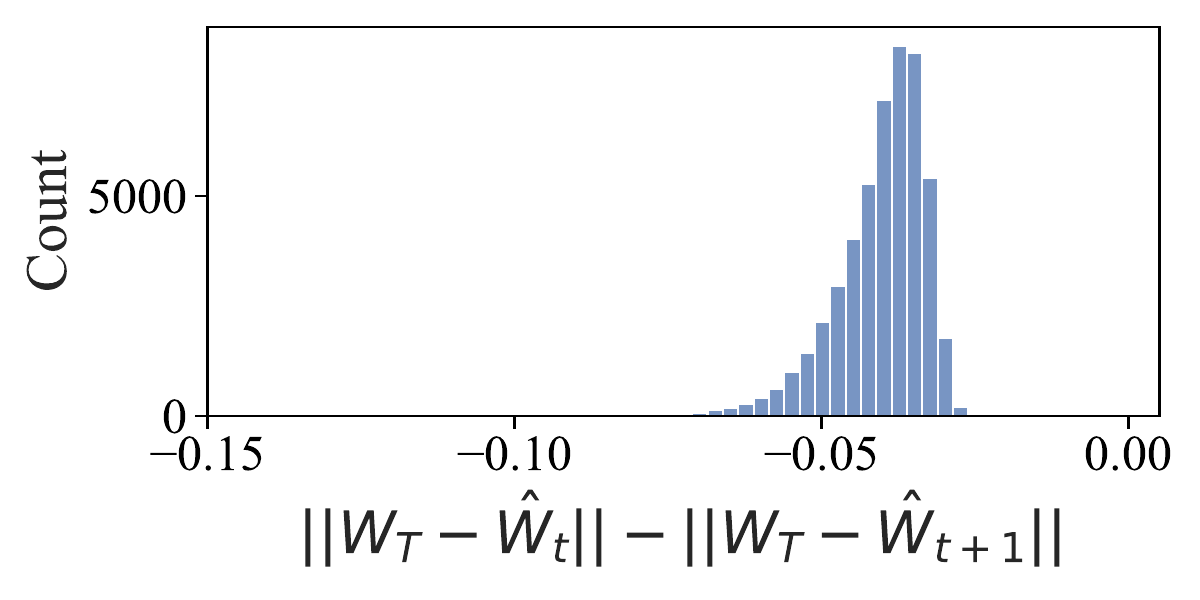}}
    \subfloat[Step 50000]{\includegraphics[width=0.3\linewidth]{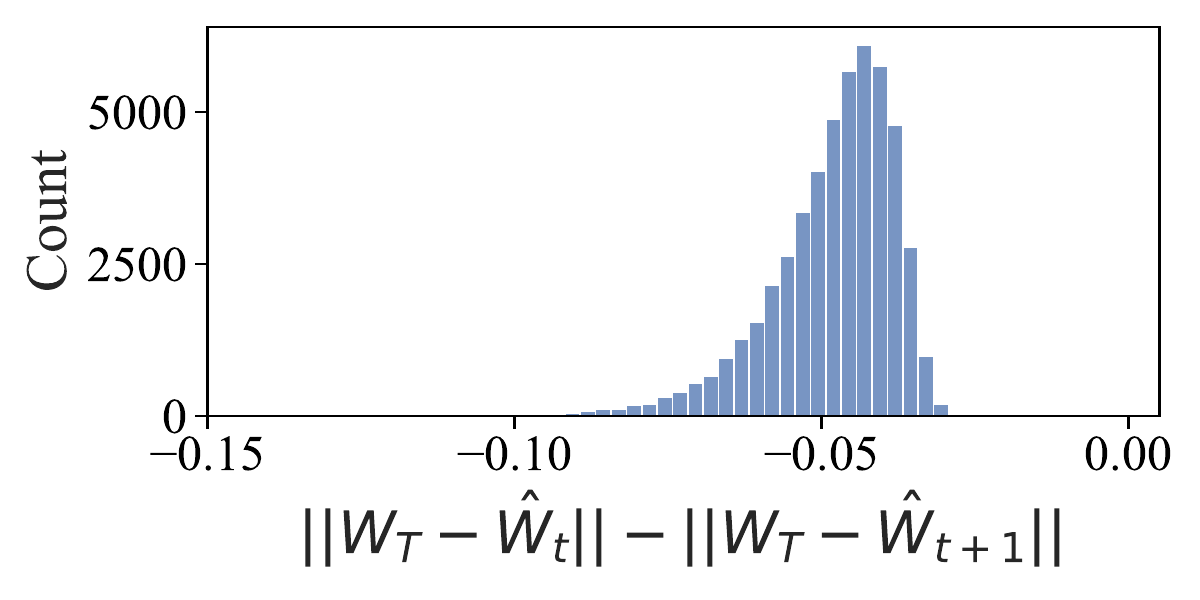}} 
    
    \subfloat[Step 60000]{\includegraphics[width=0.3\linewidth]{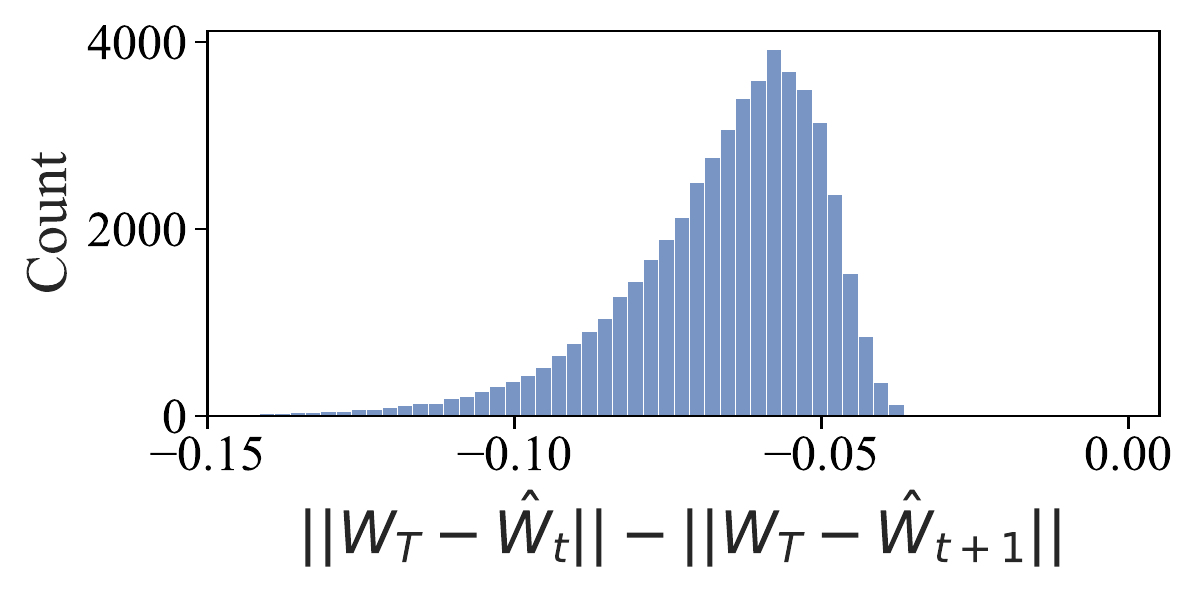}}
    \subfloat[Step 70000]{\includegraphics[width=0.3\linewidth]{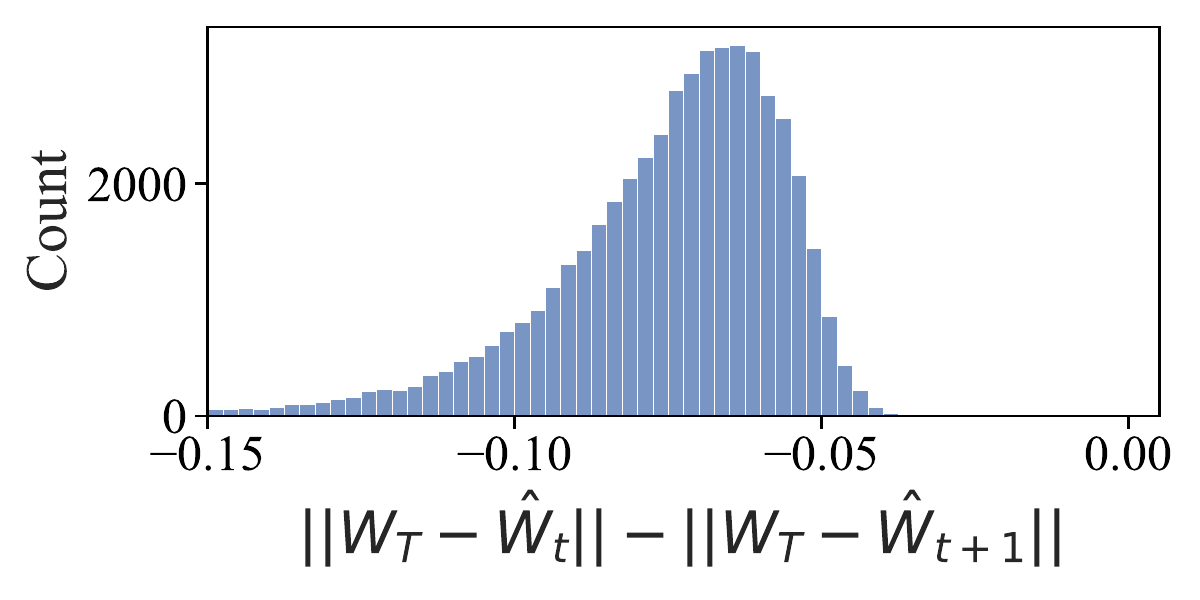}} 
    \subfloat[Training Finished (Epoch 200)]{\includegraphics[width=0.3\linewidth]{figures/hist_CIFAR100_resnet50_78125_100_dist_2_long.pdf}}
    \caption{ 
    \textbf{No individual update pushes the adversary's model $\hat{W}_t$ closer to the prover's final model $W_T$.} These are plotted in the same setting as Figure~\ref{fig:boa_cifar10_leftovers} to illustrate finer grained information on the evolution of the updates throughout training. Note that all gradients at all intermediate steps in training push the adversary's model away from $W_T$, as indicated by the negative values on the $x$-axis. We use a ResNet-50 on the CIFAR-100 dataset.
    }
    \label{fig:boa_cifar100}
\end{figure*}

\end{document}